\documentclass{article}

\usepackage[main,final]{neurips_2026}
\usepackage{algorithm}
\usepackage{algpseudocode}
\usepackage{multirow}
\usepackage{rotating}

\PassOptionsToPackage{dvipsnames}{xcolor}
\RequirePackage{xcolor}
\definecolor{navyblue}{rgb}{0.0,0.0,0.5}
\definecolor{ForestGreen}{RGB}{34,139,34}
\usepackage[colorlinks=true,urlcolor=magenta,citecolor=navyblue]{hyperref}

\usepackage{amsmath}
\usepackage{amssymb}
\usepackage{mathtools}
\usepackage{amsthm}

\usepackage[capitalize,noabbrev]{cleveref}

\theoremstyle{plain}
\newtheorem{theorem}{Theorem}[section]
\newtheorem{proposition}[theorem]{Proposition}
\newtheorem{remark}[theorem]{Remark}

\usepackage[utf8]{inputenc} 
\usepackage[T1]{fontenc}    
\usepackage{url}            
\usepackage{booktabs}       
\usepackage{amsfonts}       
\usepackage{nicefrac}       
\usepackage{microtype}      
\usepackage{xcolor}         

\usepackage{enumitem}

\usepackage{natbib}
\renewcommand{\cite}{\citep}

\input{definitions}

\title{Data-Free Reservoir Features for Efficient Long-Horizon Cold-Start Continual Learning}

\author{%
  Augustinas Jučas \\
  Department of Computer Science \\
  University of Oxford \\
  \texttt{augustinas.jucas@st-hildas.ox.ac.uk} \\
  \And
  Yangchen Pan\thanks{Part of this work was conducted while the author was at the University of Oxford.} \\
  Department of Data Science \\
  City University of Hong Kong \\
  \texttt{yangchen.pan@cityu.edu.hk}
}

\begin{document}
\maketitle

\begin{abstract}
Cold-start exemplar-free class-incremental learning requires learning a growing set of classes without replay, external pretraining, or a large initial task. Existing cold-start methods typically either train the backbone throughout the stream and compensate for semantic drift, or freeze a backbone after the first task, producing features biased toward the initial classes. These choices also create a computational tension: drift-compensation methods require repeated backbone training and increasingly expensive updates as the task horizon grows, while frozen-backbone methods are cheap but weak under cold start. We study a third option: a feature extractor that is never fit to image data at all. We propose CIRCLE, a class-incremental classifier built from fixed bidirectional two-dimensional reservoir features, adapted from BiRC2D for image classification, and streaming linear discriminant analysis heads. CIRCLE groups multiple random reservoir instantiations into feature ensembles and averages the softmax outputs of independent SLDA heads, yielding a tunable bias-variance tradeoff between richer random features and prediction-level ensembling. Because the feature extractor is fixed and the head admits streaming closed-form updates, CIRCLE performs sample-wise training without replay, task-boundary information, or backbone backpropagation. On CIFAR-100, TinyImageNet, ImageNet-Subset, and ImageNet-1k, CIRCLE is competitive at 10-20 task splits and substantially outperforms strong CS-EFCIL baselines at 50, 100, and 500 task splits, while training much faster than trained-backbone drift-compensation methods. Ablations show that our adapted BiRC2D extractor, SLDA head, and balanced ensembling each contribute to performance. Further experiments measure the effect of feature dimensionality.

\end{abstract}
\section{Introduction}

Continual learning (CL) studies models trained on non-stationary data streams, where naive updates can overwrite prior knowledge and cause \emph{catastrophic forgetting}~\cite{MCCLOSKEY1989109, FRENCH1999128}. In CL, class-incremental learning (CIL)~\cite{Zhou_2024} considers a sequence of tasks with disjoint class sets, requiring prediction over all seen classes at test time. We focus on the \emph{cold-start exemplar-free} (CS-EFCIL) setting~\cite{Zhou_2024}.


\textit{Exemplar-free} prohibits storing or replaying past data, while \textit{cold-start} further disallows pre-trained backbones or large initial tasks, preventing reliance on pre-trained representations. Hence, CS-EFCIL exposes a representation-learning tradeoff. To make this tradeoff precise, view a classifier as the composition of a feature extractor and a classification head: the feature extractor maps inputs to representations, while the head maps these representations to class predictions. Since exemplar-free methods cannot store old samples, the head often retains compact information about past classes, such as means and covariances. \textit{Most} cold-start methods continue training the feature extractor throughout the stream and compensate for semantic drift, i.e., changes in the feature space that make the head's stored old-class statistics stale~\cite{yu2020semanticdriftcompensationclassincremental}.
Such methods work well at moderate horizons, but drift control is approximate and can accumulate error over many task transitions. \textit{Alternatively}, one can freeze the feature extractor after training on the first task $\mathcal{T}_1$ and further only update a closed-form head, avoiding drift and repeated backbone training. 
However, in cold-start, the data available before freezing is only a small first task, resulting in non-transferable feature extractors biased towards  initial classes. Thus, existing approaches face a tradeoff between trained feature extractors that drift and frozen ones that are first-task biased and less transferable.


This tradeoff also has a computational dimension. Trained-backbone methods repeatedly backpropagate through the feature extractor and often remap old-class statistics after each task, making long-horizon streams potentially expensive~\cite{magistri2025efcelasticfeatureconsolidation,grzegorz2024taskrecency}. A natural third option is therefore to use a feature extractor that is frozen from the beginning and never fitted to image data. Such a feature extractor is neither pretrained externally nor biased by first-task training, and it avoids backbone backpropagation and expensive drift estimation entirely. The open question is whether such data-free features can be strong enough for nontrivial image classification.

We answer this question with \textbf{CIRCLE} (\textbf{C}lass-\textbf{I}ncremental \textbf{R}eservoir \textbf{C}lassifier with s\textbf{L}DA \textbf{E}nsemble), a cold-start exemplar-free classifier built from an \textit{untrained frozen feature extractor} and \textit{streaming linear discriminant analysis (SLDA) heads}~\cite{hayes2020lifelongmachinelearningdeep}. We adapt BiRC2D~\cite{nakanishi2024birc2d}, originally developed for spectral image anomaly detection, into such a feature extractor; ensemble multiple random instantiations  of the extractor through both feature concatenation and prediction averaging; and update each head closed-form using additive sufficient statistics. CIRCLE therefore supports sample-wise learning without replay, task-boundary information, or backbone optimization. Our contributions are the following:
\begin{itemize}[noitemsep, topsep=0pt]
    \item We study a data-free frozen-feature design for CS-EFCIL and propose CIRCLE, which combines a fixed random feature extractor with an SLDA ensemble. Our contribution is adapting and combining established reservoir and SLDA methods for image classification in long-horizon CS-EFCIL.
    \item We show that CIRCLE is competitive at CS-EFCIL on standard 10--20 task splits and considerably stronger at 50--500 task splits on CIFAR-100~\cite{krizhevsky2009learning}, TinyImageNet~\cite{wu2017tiny}, ImageNet-Subset, and ImageNet-1k~\cite{ILSVRC15}, with substantially lower training time than trained-backbone baselines and no backbone-training.
    \item We ablate the feature extractor, analytic head choice, and grouped ensemble design, showing that each component substantially contributes to performance.
\end{itemize}

\section{Related Work}

Exemplar-free class-incremental learning \cite{zhou2023pycil} can be categorized along two axes: warm- vs. cold-start settings, and trained vs. frozen feature extractors. Warm-start methods rely on external pretraining or a large initial task, whereas cold-start methods learn from comparable task sizes without external data. Independently, trained-backbone methods update the feature extractor throughout the stream, while frozen-backbone methods update only the classifier head.

\textit{Trained-backbone} methods update the feature extractor $\phi$ jointly with the head throughout the stream, so the head's old-class statistics must remain compatible with a feature space that changes over time; here warm- vs.\ cold-start differs only by pre-training. Their main limitation is \emph{semantic drift}~\cite{yu2020semanticdriftcompensationclassincremental}: as $\phi$ changes, old-class features shift and stored head statistics become outdated. Existing methods mitigate this via rehearsal/augmentation (PASS~\cite{Zhu_2021_CVPR}, IL2A~\cite{zhu2021class}), drift estimation (SDC ~\cite{yu2020semanticdriftcompensationclassincremental}, ADC~\cite{goswami2024resurrectingoldclassesnew}, LDC~\cite{gomezvilla2024exemplarfreecontinualrepresentationlearning}, AdaGauss~\cite{grzegorz2024taskrecency}, EFC~\cite{magistri2024elasticfeatureconsolidationcold}, EFC++~\cite{magistri2025efcelasticfeatureconsolidation}), or distillation-style regularization (LwF~\cite{DBLP:journals/corr/LiH16e}, AdaGauss, EFC, EFC++).
They are effective on the short and moderate horizons most commonly studied in prior CS-EFCIL work. However, because most prior evaluations focus on short horizons ($T\in\{5,10,20\}$), the behavior of these methods at much longer horizons remains underexplored. In Section \ref{sec:experiments}, we evaluate $T\in\{50,100\}$ and an ImageNet-1k $T=500$ stress test, and find that performance degrades sharply for trained-backbone drift-compensation methods in this regime. They are also computationally expensive, with per-task cost growing with the number of seen classes~\cite{magistri2025efcelasticfeatureconsolidation}.

\textit{Warm-start frozen-backbone} methods freeze the feature extractor $\phi$ after pre-training or after a large initial task $\mathcal{T}_1$. A fixed $\phi$ enables closed-form head updates over frozen features, ensuring equivalence between continual training and joint training on all seen classes. Examples include ACIL~\cite{zhuang2022acilanalyticclassincrementallearning}, DS-AL~\cite{zhuang2024dsaldualstreamanalyticlearning}, FeCAM~\cite{goswami2024fecamexploitingheterogeneityclass}, RanPAC~\cite{mcdonnell2024ranpacrandomprojectionspretrained}, SLDA~\cite{hayes2020lifelongmachinelearningdeep}, and FeTrIL~\cite{petit2023fetrilfeaturetranslationexemplarfree}. Their key requirement is a \emph{transferable} $\phi$, typically obtained via pre-training or large $\mathcal{T}_1$; without this, performance degrades. Frozen-$\phi$ methods are more efficient -- no backpropagation or drift control after $\mathcal{T}_1$ are needed, however the initial $\mathcal{T}_1$ training can still be expensive.

\textit{Cold-start frozen-backbone} methods are largely under-explored. There is little work that freezes $\phi$ under cold-start;  existing baselines, primarily built for warm-start, simply train $\phi$ on a small $\mathcal{T}_1$ then freeze $\phi$ ~\cite{goswami2024fecamexploitingheterogeneityclass}, which could hurt performance due to $\phi$'s  strong representational bias towards a small $\mathcal{T}_1$. This highlights a trade-off within \textit{cold-start} EFCIL: training $\phi$ avoids representational bias but incurs drift and cost, while freezing $\phi$ enables efficient, drift-free updates but suffers from first-task bias. An ideal cold-start approach thus would achieve a non-biased, transferable and fixed $\phi$ without pre-training, implying a backbone $\phi$ that is not fit to any task-specific data.


Reservoir computing (RC) offers such a data-free $\phi$: reservoirs are fixed random nonlinear dynamical systems that embed input \textit{sequences} into high-dimensional spaces, with theoretical guarantees, with only a trained readout~\cite{jaeger2001echo,maass2002real}. Due to recurrent nature of RC, existing RC-based CL mainly considers time-series data~\cite{bereska2022continuallearningdynamicalsystems}; image reservoirs are harder because naive pixel flattening loses spatial structure~\cite{LopezOrtiz2024,chang2019reinforcementlearningconvolutionalreservoir}. BiRC2D~\cite{nakanishi2024birc2d} addresses this by applying bidirectional 2D reservoirs over image patches. We adapt this image-reservoir idea to class-incremental image classification. To our knowledge, this is the first study of data-free, reservoir-based CS-EFCIL. 

\section{CIRCLE: Bidirectional Reservoir Ensembles with Streaming LDA}
We now describe CIRCLE. The method has three components:
(1) a fixed reservoir-based feature extractor adapted from 
BiRC2D~\citep{nakanishi2024birc2d} for classification 
(Section \ref{sec:method:birc2d}); 
(2) a feature ensembling scheme, which concatenates embeddings from several independent reservoir instantiations (Section \ref{sec:method:ensembling}); (3) a prediction ensembling scheme, which averages the probability outputs of several independent SLDA heads (Sections \ref{sec:method:ensembling} and \ref{sec:method:slda}). CIRCLE is visually depicted in Figure~\ref{fig:circle-ensembling} in the appendix.

\subsection{Data-free Reservoir Features}\label{sec:method:birc2d}

Let $\phi_s : \mathcal{X}\rightarrow\mathbb{R}^{d_r}$ denote a random reservoir-based feature extractor instantiated with a random seed $s$. The seed determines all random convolutional, reservoir, and projection weights, which remain fixed throughout the stream. We adapt BiRC2D~\cite{nakanishi2024birc2d} from image anomaly detection to image-level classification. Each image is first mapped by a small fixed random convolutional stem. The resulting feature map is divided into patches and processed by bidirectional 2D reservoir blocks, which scan spatial locations in four directions and concatenate the directional states. We then aggregate the final spatial map into a single image embedding and apply a fixed random projection to obtain the reservoir feature $\phi_s(x)$.

Compared with the original BiRC2D, our classification variant uses random convolutional preprocessing, ReLU reservoir activations, initialization chosen per dataset, spatial aggregation into an image-level vector, and a final random projection to control feature dimensionality. Figure~\ref{fig:circle-extractor} in the appendix shows the feature extractor. Full architectural details are given in Appendix~\ref{sec:appendix-circle}, and the component ablations in Section~\ref{sec:appendix-components}.

\subsection{Grouped Feature and Prediction Ensembling}
\label{sec:method:ensembling}

A single random reservoir can yield a noisy feature space, so CIRCLE uses multiple independent reservoir draws. Given $n$ random seeds, we partition them into $k$ equal-size disjoint groups, denoted $G_i, i=1,...,k$, such that $|\cup_{i=1}^k G_i| = n$ and $|G_i| = n/k = m$ (with $n,k,m$ integers by design). Feature ensembling then concatenates their embeddings: $\Phi_{G_i}(x) = ||_{s\in G_i} \phi_s(x)$ where $||_{s \in G_i}$ denotes concatenation of features generated by the seeds in $G_i$. This provides each classifier with a richer random-feature representation without training additional backbones.

CIRCLE also ensembles at the prediction level. We partition $n$ reservoirs into $k$ disjoint groups, attach one SLDA head to each group, and average their probabilities:
\[
\hat{p}(x)=\frac{1}{k}\sum_{j=1}^{k}
\mathrm{softmax}\!\left(h_j(\Phi_{G_j}(x))\right),
\]
where $h_j$ denotes the $j$th head. 
This grouped design interpolates between two extremes: using all reservoirs in one feature vector gives a rich representation but no prediction averaging, while using one reservoir per head gives many predictions but weak individual representations. 

\textbf{Bias-variance insight into the two-level ensembles.} Feature and prediction ensembling act on complementary axes, corresponding to a bias--variance tradeoff between \emph{feature richness} and \emph{ensemble diversity}. Concatenating all $n$ reservoirs yields a rich representation but a single classifier (no variance reduction), while assigning one reservoir per classifier yields diverse predictions but weak features. Grouping reservoirs into $k$ sets of size $m=n/k$ interpolates between these extremes: within-group concatenation likely reduces bias, while averaging across independent groups reduces variance. Proposition \ref{prop-biasvar} (Appendix~\ref{sec:appendix-theory}) formalizes this as a bias-variance trade-off and motivates intermediate group sizes (the optimal $m$ is typically neither $1$ nor $n$), since under an assumption on how bias decreases with group size, the bound \emph{can} be minimized by $1 < m < n$. Empirically, Section \ref{sec:exp-ablation-ensembling} confirms that balanced grouping performs best.


\subsection{Streaming LDA Head}
\label{sec:method:slda}
Each group of feature extractors is paired with a streaming linear discriminant analysis (SLDA) head~\cite{hayes2020lifelongmachinelearningdeep}. For class $c$, the head models features with mean $\mu_c$, shared covariance $\Sigma$, and empirical prior $\pi_c$. With ridge regularization $\Sigma_\lambda=\Sigma+\lambda I$, the class logit for feature $z$ is
\[
[h(z)]_c
=
z^\top \Sigma_\lambda^{-1}\mu_c
-\frac{1}{2}\mu_c^\top \Sigma_\lambda^{-1}\mu_c
+\log \pi_c .
\]

The head is updated using additive sufficient statistics: class counts $n_c$ (the number of observed samples of class $c$), class sums $T_c=\sum_{i:y_i=c}z_i\in\mathbb{R}^d$, and the shared second-moment matrix $M=\sum_i z_i z_i^\top\in\mathbb{R}^{d\times d}$ across all classes. On observing $(z,y)$, we update
\[
n_y\leftarrow n_y+1,\qquad
T_y\leftarrow T_y+z,\qquad
M\leftarrow M+zz^\top .
\]
At evaluation time, these statistics materialize the LDA classifier in closed form. 
Given fixed features and the same regularization rule, the sufficient-statistic update is order-invariant and yields the same classifier as LDA trained jointly on all seen samples.
Thus CIRCLE has no representation drift and no order-induced approximation error in the head. Full update details are given in Appendix~\ref{sec:appendix-circle}. Section~\ref{sec:exp-ablation-head} shows that SLDA performs uniformly best on top of reservoir features, compared to alternative classification methods.

\textbf{Overall Algorithm}. For each incoming sample, CIRCLE computes the feature representation for each reservoir group and updates the corresponding SLDA sufficient statistics. At test time, each head produces logits from its current closed-form LDA parameters; probabilities are averaged across heads. Detailed pseudocode is provided in Appendix~\ref{sec:appendix-circle}.

\section{Experiments}\label{sec:experiments}

Our experiments test four questions within the cold-start EFCIL regime. First, how does CIRCLE compare with trained-backbone and freeze-after-first-task baselines as the task horizon increases? Second, does the use of many untrained reservoir draws make CIRCLE computationally expensive? Third, we ask whether the same fixed-feature analytic design remains viable in an extreme long-horizon setting, using a $500$-task ImageNet-1k stream.  Finally, which components of CIRCLE are necessary: the BiRC2D-style feature extractor, the SLDA head, and the grouped feature/prediction ensemble?

We first describe the common setup, then present the main comparison and computation study in Section~\ref{sec:exp-main}, the ImageNet-1k long-horizon stress test in Section~\ref{sec:exp-imagenet500}, and ablations in Section~\ref{sec:exp-ablations}. 

\textbf{Datasets and protocol.}
We evaluate on CIFAR-100, TinyImageNet, and ImageNet-Subset with time horizon $T \in \{10,20,50,100\}$, and evaluate additional long-horizon $T=500$ experiments on full ImageNet-1k. All experiments follow the cold-start exemplar-free setting: no replay, no pre-training, and classes split evenly across $T$ tasks. Results are averaged over 5 seeds unless specified. 

\textbf{Baselines.}
We compare against representative EFCIL methods across three categories: 
\emph{drift-compensation / distillation} (EFC++~\citep{magistri2025efcelasticfeatureconsolidation}, AdaGauss~\citep{grzegorz2024taskrecency}, ADC~\citep{goswami2024resurrectingoldclassesnew}, LwF~\cite{DBLP:journals/corr/LiH16e}), 
\emph{prototype rehearsal / augmentation} (PASS~\citep{Zhu_2021_CVPR}, IL2A~\citep{zhu2021class}), and 
\emph{frozen-backbone analytic methods} (ACIL~\citep{zhuang2022acilanalyticclassincrementallearning}, DS-AL~\citep{zhuang2024dsaldualstreamanalyticlearning}, FeCAM~\citep{goswami2024fecamexploitingheterogeneityclass}, FeTrIL~\citep{petit2023fetrilfeaturetranslationexemplarfree}), adapted to cold-start by training on $\mathcal{T}_1$ and freezing thereafter. All trained-backbone baselines use ResNet-18. 


We report mean incremental accuracy $\bar{A}$ and final accuracy $A_T$. For efficiency, we report end-to-end wall-clock training time, total parameter count, and learnable parameter count. All methods are tuned separately for each $(T,\mathrm{dataset})$ configuration, which is important because the long-horizon splits induce substantially different task dynamics. Tuning protocol, selected hyperparameters, and implementation details are given in Appendices~\ref{sec:appendix-experiment-details} and~\ref{sec:hyperparameters}.

\subsection{Main Comparison and Computational Efficiency}
\label{sec:exp-main}

\subsubsection{Performance for T = 10, 20, 50, 100}

Table~\ref{tab:main-comparison} compares CIRCLE with all baselines on CIFAR-100, TinyImageNet, and ImageNet-Subset across task horizons \(T \in \{10,20,50,100\}\).

\begin{table}[t]
\caption{Cold-start EFCIL results on CIFAR-100, TinyImageNet, and ImageNet-Subset for $T \in \{10, 20, 50, 100\}$. Each cell reports mean $\pm$ std over 5 seeds for mean incremental accuracy ($\bar{A}$) and final accuracy ($A_T$). Best mean per data setting in \textbf{bold}, second-best \underline{underlined}. For our CIRCLE, $A_T$ is invariant across task splits because continual training for CIRCLE is equivalent to joint training.}
\label{tab:main-comparison}
\scriptsize
\setlength{\tabcolsep}{4.42pt}
\centering
\begin{tabular}{llcccccccc}
\toprule
 & Algorithm & \multicolumn{2}{c}{T=10} & \multicolumn{2}{c}{T=20} & \multicolumn{2}{c}{T=50} & \multicolumn{2}{c}{T=100} \\
\cmidrule(lr){3-4}\cmidrule(lr){5-6}\cmidrule(lr){7-8}\cmidrule(lr){9-10}
 &  & \(\bar{A}\) & \(A_T\) & \(\bar{A}\) & \(A_T\) & \(\bar{A}\) & \(A_T\) & \(\bar{A}\) & \(A_T\) \\
\midrule
\multirow{11}{*}{\rotatebox[origin=c]{90}{CIFAR-100}} & AdaGauss & \underline{62.53 \ensuremath{\pm} 0.2} & \underline{46.57 \ensuremath{\pm} 0.4} & \textbf{56.11 \ensuremath{\pm} 0.3} & 40.80 \ensuremath{\pm} 0.3 & \underline{40.30 \ensuremath{\pm} 0.6} & 15.64 \ensuremath{\pm} 2.5 & 21.26 \ensuremath{\pm} 0.4 & 6.87 \ensuremath{\pm} 0.3 \\
 & EFC++ & \textbf{66.61 \ensuremath{\pm} 0.2} & \textbf{52.68 \ensuremath{\pm} 0.4} & 55.74 \ensuremath{\pm} 0.1 & \underline{42.10 \ensuremath{\pm} 0.5} & 38.93 \ensuremath{\pm} 0.2 & \underline{26.95 \ensuremath{\pm} 0.4} & 26.73 \ensuremath{\pm} 0.6 & \underline{16.40 \ensuremath{\pm} 0.3} \\
 & ADC & 61.98 \ensuremath{\pm} 0.2 & 44.71 \ensuremath{\pm} 0.2 & 49.86 \ensuremath{\pm} 0.4 & 31.89 \ensuremath{\pm} 0.5 & 30.57 \ensuremath{\pm} 0.3 & 16.22 \ensuremath{\pm} 0.4 & 8.11 \ensuremath{\pm} 0.1 & 3.35 \ensuremath{\pm} 0.2 \\
 & ACIL & 55.36 \ensuremath{\pm} 0.5 & 40.73 \ensuremath{\pm} 0.4 & 44.54 \ensuremath{\pm} 0.4 & 33.16 \ensuremath{\pm} 0.5 & 37.37 \ensuremath{\pm} 0.4 & 24.95 \ensuremath{\pm} 0.6 & 23.13 \ensuremath{\pm} 0.6 & 14.04 \ensuremath{\pm} 0.6 \\
 & FeCAM & 49.71 \ensuremath{\pm} 0.3 & 33.08 \ensuremath{\pm} 0.4 & 37.02 \ensuremath{\pm} 0.6 & 23.09 \ensuremath{\pm} 0.5 & 34.05 \ensuremath{\pm} 0.4 & 20.97 \ensuremath{\pm} 0.4 & 23.44 \ensuremath{\pm} 1.1 & 13.85 \ensuremath{\pm} 0.7 \\
 & FeTRIL & 52.06 \ensuremath{\pm} 0.2 & 36.60 \ensuremath{\pm} 0.3 & 42.25 \ensuremath{\pm} 1.4 & 27.71 \ensuremath{\pm} 1.2 & 32.76 \ensuremath{\pm} 0.2 & 20.24 \ensuremath{\pm} 0.2 & 22.15 \ensuremath{\pm} 0.2 & 12.04 \ensuremath{\pm} 0.2 \\
 & IL2A & 43.23 \ensuremath{\pm} 4.3 & 27.24 \ensuremath{\pm} 3.3 & 15.70 \ensuremath{\pm} 2.1 & 5.04 \ensuremath{\pm} 1.9 & 26.97 \ensuremath{\pm} 4.9 & 15.31 \ensuremath{\pm} 3.8 & 13.86 \ensuremath{\pm} 0.8 & 4.43 \ensuremath{\pm} 0.6 \\
 & PASS & 52.85 \ensuremath{\pm} 0.3 & 36.76 \ensuremath{\pm} 0.3 & 40.73 \ensuremath{\pm} 0.6 & 24.85 \ensuremath{\pm} 0.4 & 33.76 \ensuremath{\pm} 0.6 & 19.23 \ensuremath{\pm} 0.3 & \underline{28.70 \ensuremath{\pm} 0.6} & 15.48 \ensuremath{\pm} 0.1 \\
 & LwF & 50.92 \ensuremath{\pm} 0.4 & 28.92 \ensuremath{\pm} 0.4 & 37.15 \ensuremath{\pm} 0.5 & 18.06 \ensuremath{\pm} 0.4 & 19.41 \ensuremath{\pm} 0.4 & 6.03 \ensuremath{\pm} 0.5 & 5.13 \ensuremath{\pm} 0.3 & 1.02 \ensuremath{\pm} 0.1 \\
 & DS-AL & 56.43 \ensuremath{\pm} 0.2 & 41.76 \ensuremath{\pm} 0.2 & 45.79 \ensuremath{\pm} 0.1 & 34.89 \ensuremath{\pm} 0.3 & 38.17 \ensuremath{\pm} 0.5 & 25.84 \ensuremath{\pm} 0.5 & 23.86 \ensuremath{\pm} 0.8 & 15.02 \ensuremath{\pm} 0.5 \\
\cmidrule(lr){2-10}
 & CIRCLE & 55.32 \ensuremath{\pm} 0.3 & 45.34 \ensuremath{\pm} 0.3 & \underline{55.90 \ensuremath{\pm} 0.3} & \textbf{45.34 \ensuremath{\pm} 0.3} & \textbf{56.97 \ensuremath{\pm} 0.3} & \textbf{45.35 \ensuremath{\pm} 0.3} & \textbf{57.22 \ensuremath{\pm} 0.3} & \textbf{45.34 \ensuremath{\pm} 0.3} \\
\midrule
\multirow{11}{*}{\rotatebox[origin=c]{90}{TinyImageNet}} & AdaGauss & \underline{50.67 \ensuremath{\pm} 0.2} & \underline{37.97 \ensuremath{\pm} 0.4} & \underline{44.46 \ensuremath{\pm} 0.2} & \underline{31.44 \ensuremath{\pm} 0.2} & \underline{37.02 \ensuremath{\pm} 0.4} & 21.94 \ensuremath{\pm} 0.3 & 24.56 \ensuremath{\pm} 0.3 & 9.71 \ensuremath{\pm} 0.2 \\
 & EFC++ & \textbf{52.34 \ensuremath{\pm} 0.3} & \textbf{39.49 \ensuremath{\pm} 0.3} & \textbf{47.51 \ensuremath{\pm} 0.1} & \textbf{35.16 \ensuremath{\pm} 0.4} & 36.24 \ensuremath{\pm} 0.2 & \underline{24.37 \ensuremath{\pm} 0.4} & \underline{25.17 \ensuremath{\pm} 0.1} & \underline{13.95 \ensuremath{\pm} 0.3} \\
 & ADC & 47.54 \ensuremath{\pm} 0.2 & 31.40 \ensuremath{\pm} 0.3 & 37.21 \ensuremath{\pm} 0.1 & 21.14 \ensuremath{\pm} 0.2 & 26.54 \ensuremath{\pm} 0.2 & 13.05 \ensuremath{\pm} 0.3 & 18.59 \ensuremath{\pm} 0.3 & 9.29 \ensuremath{\pm} 0.4 \\
 & ACIL & 39.86 \ensuremath{\pm} 0.3 & 27.26 \ensuremath{\pm} 0.2 & 31.34 \ensuremath{\pm} 0.1 & 20.88 \ensuremath{\pm} 0.2 & 23.52 \ensuremath{\pm} 0.3 & 14.64 \ensuremath{\pm} 0.3 & 16.53 \ensuremath{\pm} 1.2 & 8.47 \ensuremath{\pm} 0.8 \\
 & FeCAM & 48.02 \ensuremath{\pm} 0.3 & 34.33 \ensuremath{\pm} 0.2 & 40.26 \ensuremath{\pm} 0.2 & 27.81 \ensuremath{\pm} 0.2 & 25.56 \ensuremath{\pm} 0.5 & 14.57 \ensuremath{\pm} 0.2 & 16.36 \ensuremath{\pm} 0.4 & 8.14 \ensuremath{\pm} 0.1 \\
 & FeTRIL & 37.02 \ensuremath{\pm} 0.2 & 22.72 \ensuremath{\pm} 0.3 & 28.10 \ensuremath{\pm} 0.6 & 15.37 \ensuremath{\pm} 0.7 & 22.26 \ensuremath{\pm} 0.2 & 11.96 \ensuremath{\pm} 0.3 & 16.11 \ensuremath{\pm} 0.1 & 8.13 \ensuremath{\pm} 0.2 \\
 & IL2A & 37.72 \ensuremath{\pm} 0.3 & 20.30 \ensuremath{\pm} 0.4 & 27.13 \ensuremath{\pm} 0.2 & 9.99 \ensuremath{\pm} 0.4 & 23.94 \ensuremath{\pm} 0.2 & 13.05 \ensuremath{\pm} 0.4 & 16.57 \ensuremath{\pm} 0.4 & 7.39 \ensuremath{\pm} 0.1 \\
 & PASS & 38.71 \ensuremath{\pm} 0.3 & 25.88 \ensuremath{\pm} 0.4 & 31.05 \ensuremath{\pm} 0.9 & 18.38 \ensuremath{\pm} 0.6 & 24.10 \ensuremath{\pm} 0.3 & 12.92 \ensuremath{\pm} 0.4 & 19.57 \ensuremath{\pm} 0.4 & 8.30 \ensuremath{\pm} 0.4 \\
 & LwF & 37.11 \ensuremath{\pm} 0.4 & 20.61 \ensuremath{\pm} 0.4 & 26.37 \ensuremath{\pm} 0.1 & 11.16 \ensuremath{\pm} 0.1 & 18.53 \ensuremath{\pm} 0.2 & 5.96 \ensuremath{\pm} 0.3 & 8.88 \ensuremath{\pm} 0.1 & 2.32 \ensuremath{\pm} 0.1 \\
 & DS-AL & 37.92 \ensuremath{\pm} 0.3 & 26.38 \ensuremath{\pm} 0.2 & 30.66 \ensuremath{\pm} 0.3 & 20.94 \ensuremath{\pm} 0.2 & 22.81 \ensuremath{\pm} 0.5 & 14.22 \ensuremath{\pm} 0.4 & 18.54 \ensuremath{\pm} 0.2 & 10.11 \ensuremath{\pm} 0.2 \\
\cmidrule(lr){2-10}
 & CIRCLE & 39.72 \ensuremath{\pm} 0.2 & 30.77 \ensuremath{\pm} 0.2 & 40.91 \ensuremath{\pm} 0.2 & 30.78 \ensuremath{\pm} 0.2 & \textbf{41.78 \ensuremath{\pm} 0.2} & \textbf{30.77 \ensuremath{\pm} 0.2} & \textbf{42.17 \ensuremath{\pm} 0.2} & \textbf{30.75 \ensuremath{\pm} 0.2} \\
\midrule
\multirow{11}{*}{\rotatebox[origin=c]{90}{ImageNet-Subset}} & AdaGauss & \underline{62.53 \ensuremath{\pm} 0.3} & \underline{45.84 \ensuremath{\pm} 0.4} & \underline{53.76 \ensuremath{\pm} 0.4} & 33.87 \ensuremath{\pm} 0.7 & 33.76 \ensuremath{\pm} 0.3 & 14.28 \ensuremath{\pm} 0.3 & 22.02 \ensuremath{\pm} 0.3 & 6.93 \ensuremath{\pm} 0.3 \\
 & EFC++ & \textbf{67.97 \ensuremath{\pm} 0.4} & \textbf{53.46 \ensuremath{\pm} 0.4} & \textbf{58.54 \ensuremath{\pm} 0.6} & \textbf{40.86 \ensuremath{\pm} 0.6} & \underline{40.45 \ensuremath{\pm} 0.3} & \underline{23.10 \ensuremath{\pm} 0.6} & \underline{28.47 \ensuremath{\pm} 0.6} & 13.92 \ensuremath{\pm} 0.5 \\
 & ADC & 58.36 \ensuremath{\pm} 0.4 & 37.44 \ensuremath{\pm} 0.8 & 49.32 \ensuremath{\pm} 0.4 & 28.62 \ensuremath{\pm} 0.4 & 29.10 \ensuremath{\pm} 0.6 & 12.87 \ensuremath{\pm} 0.1 & 7.42 \ensuremath{\pm} 0.2 & 1.02 \ensuremath{\pm} 0.0 \\
 & ACIL & 59.77 \ensuremath{\pm} 0.2 & 44.46 \ensuremath{\pm} 0.3 & 45.70 \ensuremath{\pm} 0.2 & 32.18 \ensuremath{\pm} 0.3 & 36.60 \ensuremath{\pm} 0.5 & 22.09 \ensuremath{\pm} 0.3 & 26.39 \ensuremath{\pm} 0.6 & 14.96 \ensuremath{\pm} 0.5 \\
 & FeCAM & 58.38 \ensuremath{\pm} 0.6 & 42.32 \ensuremath{\pm} 0.7 & 44.86 \ensuremath{\pm} 0.4 & 29.24 \ensuremath{\pm} 0.6 & 21.39 \ensuremath{\pm} 0.4 & 9.88 \ensuremath{\pm} 0.2 & 27.22 \ensuremath{\pm} 0.4 & 14.20 \ensuremath{\pm} 0.5 \\
 & FeTRIL & 55.60 \ensuremath{\pm} 0.5 & 39.46 \ensuremath{\pm} 0.7 & 44.82 \ensuremath{\pm} 0.2 & 28.94 \ensuremath{\pm} 0.7 & 31.14 \ensuremath{\pm} 0.3 & 16.98 \ensuremath{\pm} 0.1 & 24.02 \ensuremath{\pm} 0.4 & 11.56 \ensuremath{\pm} 0.4 \\
 & IL2A & 48.55 \ensuremath{\pm} 0.7 & 29.52 \ensuremath{\pm} 1.0 & 41.76 \ensuremath{\pm} 0.4 & 23.73 \ensuremath{\pm} 1.4 & 22.85 \ensuremath{\pm} 0.5 & 10.07 \ensuremath{\pm} 0.4 & 13.71 \ensuremath{\pm} 1.4 & 2.11 \ensuremath{\pm} 0.7 \\
 & PASS & 49.72 \ensuremath{\pm} 4.4 & 33.73 \ensuremath{\pm} 2.0 & 40.62 \ensuremath{\pm} 1.8 & 23.18 \ensuremath{\pm} 1.6 & 25.95 \ensuremath{\pm} 2.5 & 8.15 \ensuremath{\pm} 1.2 & 18.87 \ensuremath{\pm} 6.0 & 5.31 \ensuremath{\pm} 1.9 \\
 & LwF & 56.97 \ensuremath{\pm} 0.3 & 29.34 \ensuremath{\pm} 0.9 & 40.89 \ensuremath{\pm} 0.4 & 14.88 \ensuremath{\pm} 0.2 & 18.11 \ensuremath{\pm} 0.4 & 3.74 \ensuremath{\pm} 0.4 & 4.79 \ensuremath{\pm} 0.5 & 0.73 \ensuremath{\pm} 0.2 \\
 & DS-AL & 60.48 \ensuremath{\pm} 0.4 & 45.56 \ensuremath{\pm} 0.5 & 46.68 \ensuremath{\pm} 0.3 & 32.91 \ensuremath{\pm} 0.2 & 37.57 \ensuremath{\pm} 0.3 & 23.02 \ensuremath{\pm} 0.3 & 26.71 \ensuremath{\pm} 0.7 & \underline{15.35 \ensuremath{\pm} 0.5} \\
\cmidrule(lr){2-10}
 & CIRCLE & 52.18 \ensuremath{\pm} 0.4 & 39.68 \ensuremath{\pm} 0.3 & 53.54 \ensuremath{\pm} 0.4 & \underline{39.66 \ensuremath{\pm} 0.3} & \textbf{54.56 \ensuremath{\pm} 0.4} & \textbf{39.65 \ensuremath{\pm} 0.3} & \textbf{54.71 \ensuremath{\pm} 0.4} & \textbf{39.66 \ensuremath{\pm} 0.3} \\
\bottomrule
\end{tabular}
\label{tab:final-seeded-mean-std}
\end{table}


At \(T=10\), drift-compensation methods (EFC++, AdaGauss) perform best across all datasets, as limited task transitions keep drift-estimation errors small. CIRCLE does not match them in mean incremental accuracy \(\bar{A}\), but, despite no image-trained backbone, remains competitive with other cold-start baselines, often matching or exceeding ACIL, FeCAM, FeTRIL, PASS, and LwF on at least one metric. At \(T=20\), CIRCLE’s relative performance improves. It achieves the best final accuracy \(A_T\) on CIFAR-100 and second-best \(\bar{A}\) (just below AdaGauss); and ranks second in \(A_T\) on ImageNet-Subset while staying close in \(\bar{A}\). TinyImageNet is less favourable: EFC++ and AdaGauss lead, but CIRCLE still outperforms the other baselines in \(A_T\) and remains competitive in \(\bar{A}\). 

As the task horizon grows, trained-backbone methods (e.g., EFC++, AdaGauss) degrade sharply, with accuracies dropping by around 10-30 percentage points from $T=10$, potentially due to accumulated drift-estimation errors. Freeze-after-first-task methods also deteriorate, as the initial task likely becomes too small to yield transferable features, inducing representation bias. CIRCLE avoids both failure modes: its feature extractor is fixed and never fitted to the first task, and its SLDA head updates via additive sufficient statistics, making continual training equivalent to batch training. As a direct result, \(A_T\) is invariant to task splits. The slight increase in \(\bar{A}\) as $T$ increases for CIRCLE is solely a measuring artefact - this result comes from measuring accuracy at different points: CIFAR-100 at $T=10$ averages after 10, 20, \ldots, 100 classes, whereas $T=100$ measures accuracy after every class (see Appendix~\ref{sec:appendix-evaluation-grid}). CIRCLE is strongest across all three datasets at \(T=50\) and \(T=100\).

\textbf{Takeaway}. The comparison shows that relative performance changes with the task horizon. Drift-compensation methods are strongest at short horizons, where few feature-space shifts have occurred. As the number of task transitions grows, their advantage disappears and then reverses. CIRCLE is therefore not a short-horizon universal SOTA method; its advantage is specifically in long-horizon cold-start EFCIL, where avoiding both representation drift and first-task feature bias becomes decisive.

\subsubsection{Training Time and Computation Cost}\label{sec:exp-time}

To assess computation cost, we compare CIRCLE against representative strong baselines from different design families: EFC++ and AdaGauss as trained-backbone drift-compensation methods, ACIL and FeCAM as freeze-after-first-task analytic methods, and FeTRIL as a freeze-after-first-task method with an SGD-trained classification head. Table~\ref{tab:time} reports wall-clock training time and parameter counts at \(T=20\).


\begin{table}[ht]
\centering
\caption{Wall-clock training time at \(T=20\), mean \(\pm\) std over 3 seeds, combined with parameter counts. Times are in minutes, parameter counts are in millions. \textit{Total} denotes the total parameter count, \textit{Learn} denotes the learnable parameter count. Fastest method in \textbf{bold}, second-fastest \underline{underlined}.}
\label{tab:time}
\scriptsize
\setlength{\tabcolsep}{5.5pt}
\begin{tabular}{lccc  ccc  ccc}
\toprule
& \multicolumn{3}{c}{CIFAR-100} & \multicolumn{3}{c}{TinyImageNet} & \multicolumn{3}{c}{ImageNet-Subset} \\
\cmidrule(lr){2-4}\cmidrule(lr){5-7}\cmidrule(lr){8-10}
Method & Time & Total & Learn. & Time & Total & Learn. & Time & Total & Learn. \\
\midrule
ACIL     & \underline{9.27 \ensuremath{\pm} 0.05} & 16.18M & 11.99M & 52.89 \ensuremath{\pm} 0.12 & 17.10M & 12.90M & \underline{87.26 \ensuremath{\pm} 0.43} & 21.20M & 12.81M \\
FeCAM    & 15.53 \ensuremath{\pm} 0.24 & 11.22M & 11.22M & \underline{26.11 \ensuremath{\pm} 0.37} & 11.27M & 11.27M & 107.07 \ensuremath{\pm} 1.70 & 11.22M & 11.22M \\
FeTrIL   & 177.62 \ensuremath{\pm} 8.96 & 11.31M & 11.31M & 560.09 \ensuremath{\pm} 10.46 & 11.37M & 11.37M & 288.70 \ensuremath{\pm} 0.68 & 11.32M & 11.32M \\
AdaGauss & 119.63 \ensuremath{\pm} 4.07 & 11.20M & 11.20M & 200.17 \ensuremath{\pm} 1.57 & 11.20M & 11.20M & 858.50 \ensuremath{\pm} 9.59 & 11.21M & 11.21M \\
EFC++    & 98.39 \ensuremath{\pm} 2.66 & 11.22M & 11.22M & 310.04 \ensuremath{\pm} 5.04 & 11.27M & 11.27M & 390.31 \ensuremath{\pm} 6.44 & 11.23M & 11.23M \\
\midrule
CIRCLE   & \textbf{2.04 \ensuremath{\pm} 0.03} & 11.69M & 7.04M & \textbf{11.51 \ensuremath{\pm} 0.23} & 11.46M & 10.04M & \textbf{39.22 \ensuremath{\pm} 0.17} & 11.05M & 6.48M \\
\bottomrule
\end{tabular}
\end{table}

Despite using multiple reservoirs, CIRCLE is the fastest method on all three datasets: 4.5$\times$ faster than the second-fastest method on CIFAR-100, 2.3$\times$ on TinyImageNet, and 2.2$\times$ on ImageNet-Subset. Relative to the faster drift-compensation baseline on each dataset, the gap is roughly an order of magnitude. CIRCLE also uses fewer learnable parameters than all compared methods, while keeping total parameter count comparable to the ResNet-18-based baselines. Thus, the ensemble is not a hidden training-cost advantage: the additional reservoirs are fixed random draws, not separately trained backbones.

Figure~\ref{fig:duration-accuracy-tradeoff} (Appendix \ref{sec:appendix-wallclocktimefig}) provides an additional view of the accuracy--time trade-off, which shows that CIRCLE can finish the whole stream before some baselines finish early training.

\textbf{Takeaway.} The reservoir ensemble does not make CIRCLE expensive to train. Although CIRCLE may use many fixed reservoir draws, these draws require no optimization. Training is therefore dominated by forward passes and additive SLDA statistic updates, not backbone backpropagation.


\subsection{Long-horizon Behaviour: ImageNet-1k at T=500}
\label{sec:exp-imagenet500}

We use ImageNet-1k with \(T=500\) as an extreme long-horizon stress test for CIRCLE. Each task contains only two classes, so the method must accumulate information through many small updates while retaining early classes. This setting directly tests whether fixed data-free reservoir features with analytic SLDA updates remain useful beyond the standard short- and moderate-horizon regimes.

We compare CIRCLE with representative methods from the two main baseline families: EFC++ and AdaGauss as trained-backbone methods, and ACIL, DS-AL, and FeCAM as frozen-feature analytic methods. To complement incremental accuracy, we also track accuracy on the first task \(\mathcal{T}_1\). This gives a metric of whether the earliest class statistics remain useful after hundreds of later updates.

Figure~\ref{fig:imagenet-t500-stream-accuracy} shows that CIRCLE remains stable throughout the stream and achieves the highest incremental accuracy at every task index. Its first-task accuracy never reaches zero, indicating that early-class statistics remain usable even after hundreds of updates, which matches CIRCLE's construction of continual training being equivalent to joint training.

The trained-backbone baselines are less stable in this regime: EFC++ and AdaGauss reach zero first-task accuracy after roughly \(30\) and \(75\) tasks, respectively. They are also difficult to run at this horizon: EFC++ is truncated after a near \(5\)-day budget, while AdaGauss fails after about \(110\) tasks due to numerical instability. Thus, the \(T=500\) experiment supports the main long-horizon claim of CIRCLE: fixed reservoir features with analytic updates remain accurate and practical when the task stream becomes very long.

\begin{figure}[ht]
\centering
\begin{minipage}[t]{0.46\linewidth}
\vspace{0pt}
  \centering
  \includegraphics[width=\linewidth]{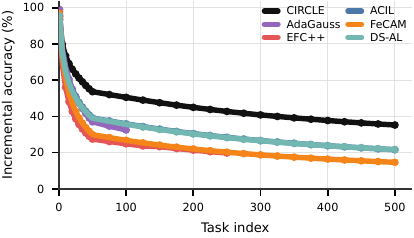}
\end{minipage}\hfill
\begin{minipage}[t]{0.45\linewidth}
\vspace{0pt}
  \centering
  \includegraphics[width=\linewidth]{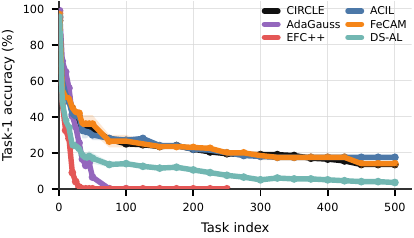}
\end{minipage}

\caption{ImageNet-1k at \(T=500\). Accuracy is tracked over the task stream and averaged over \(2\) seeds (single std shaded). EFC++ is truncated after a near \(5\)-day budget, and AdaGauss after failing at \(\sim110\) tasks due to numerical instability. Left: incremental test accuracy. Right: mean accuracy tracked for the first task, showing how quickly each method forgets the start of the stream.}
\label{fig:imagenet-t500-stream-accuracy}
\end{figure}




\textbf{Takeaway.} CIRCLE remains stable in very long task horizons, outperforming both frozen- and trained- backbone baselines, while trained-backbone baselines collapse in such a regime.

\subsection{Ablations}
\label{sec:exp-ablations}

We isolate the contributions of CIRCLE's three main design choices: the reservoir feature extractor, the analytic classification head, and the two-level ensembling scheme. All ablations are performed at the \(T=20\) split on each dataset, using \(5\) seeds unless otherwise specified. Unless stated otherwise, all components not under study are kept fixed to the default CIRCLE configuration. Hyperparameters specific to each variant, including feature- and prediction-ensembling sizes, are tuned separately for every variant using the same protocol as in the main experiments.

\subsubsection{Feature Extractor}
\label{sec:exp-ablation-extractor}

We first evaluate whether BiRC2D is essential, or whether similar performance can be obtained with other non-trained feature extractors. Therefore, we replace the BiRC2D extractor with six alternatives:
(i) a randomly-initialised ResNet-18 with BatchNorm, matching the architecture used by the trained-backbone baselines;
(ii) a randomly-initialised ResNet-18 with LayerNorm instead of BatchNorm, controlling for sensitivity of BatchNorm at random initialisation;
(iii) a randomly-initialised VGG-13, a comparable-capacity convolutional architecture without normalization layers;
(iv) Patch-RNN, a natural extension of pixel-flatten-into-ESN approaches~\citep{LopezOrtiz2024}, where the image is divided into patches and processed sequentially by a single ESN;
(v) Conv+Patch-RNN, which adds a small fixed random convolutional stack before the Patch-RNN, in the spirit of \citet{chang2019reinforcementlearningconvolutionalreservoir};
and (vi) the scattering transform, an untrained, theory-driven CNN-style filter bank~\citep{DBLP:journals/corr/abs-1203-1513}. Additional details on these extractors are given in Table~\ref{tab:app-feature-extractor-baselines}.

For each extractor, we keep the rest of the CIRCLE pipeline unchanged: features are ensembled, followed by SLDA heads, and combined through prediction ensembling. Ensembling sizes and extractor-specific hyperparameters are tuned per variant. Table~\ref{tab:ablation-extractor} reports the results at \(T=20\) on all three benchmarks.

\begin{table}[ht]
\centering
\small
\caption{Feature extractor ablation at $T=20$ on CIFAR-100, TinyImageNet, and ImageNet-Subset. All extractors are non-trained and combined with SLDA heads; ensembling sizes and any extractor-specific hyperparameters are tuned per variant. Mean $\pm$ std over 5 seeds.}
\label{tab:ablation-extractor}
\scriptsize
\begin{tabular}{lcccccc}
\toprule
 & \multicolumn{2}{c}{CIFAR-100} & \multicolumn{2}{c}{TinyImageNet} & \multicolumn{2}{c}{ImageNet-Subset} \\
\cmidrule(lr){2-3}\cmidrule(lr){4-5}\cmidrule(lr){6-7}
Extractor & $\bar A$& $A_T$ & $\bar A$& $A_T$ &$\bar A$& $A_T$  \\
\midrule
Random ResNet-18 (BN) & 39.5 \ensuremath{\pm} 0.2 & \underline{31.0 \ensuremath{\pm} 0.2} & 18.4 \ensuremath{\pm} 0.3 & 11.7 \ensuremath{\pm} 0.1 & 38.7 \ensuremath{\pm} 0.4 & \underline{28.3 \ensuremath{\pm} 0.5} \\
Random ResNet-18 (LN) & \underline{40.2 \ensuremath{\pm} 0.2} & 28.8 \ensuremath{\pm} 0.2 & 21.5 \ensuremath{\pm} 0.2 & 14.0 \ensuremath{\pm} 0.1 & \underline{41.1 \ensuremath{\pm} 0.3} & 27.4 \ensuremath{\pm} 0.3 \\
Random VGG-13 & 40.0 \ensuremath{\pm} 0.2 & 28.3 \ensuremath{\pm} 0.2 & \underline{26.8 \ensuremath{\pm} 0.2} & \underline{17.5 \ensuremath{\pm} 0.2} & 38.5 \ensuremath{\pm} 0.2 & 24.2 \ensuremath{\pm} 0.2 \\
Patch-RNN & 21.4 \ensuremath{\pm} 0.1 & 13.6 \ensuremath{\pm} 0.1 & 9.7 \ensuremath{\pm} 0.1 & 5.0 \ensuremath{\pm} 0.1 & 12.7 \ensuremath{\pm} 0.2 & 5.9 \ensuremath{\pm} 0.1 \\
Conv + Patch-RNN & 26.9 \ensuremath{\pm} 0.1 & 17.4 \ensuremath{\pm} 0.1 & 12.6 \ensuremath{\pm} 0.1 & 6.5 \ensuremath{\pm} 0.1 & 15.7 \ensuremath{\pm} 0.4 & 7.3 \ensuremath{\pm} 0.3 \\
Scattering transform & 34.7 \ensuremath{\pm} 0.0 & 22.9 \ensuremath{\pm} 0.0 & 23.2 \ensuremath{\pm} 0.0 & 14.4 \ensuremath{\pm} 0.0 & 28.9 \ensuremath{\pm} 0.0 & 16.2 \ensuremath{\pm} 0.1 \\
\midrule
BiRC2D (CIRCLE) & \textbf{55.9 \ensuremath{\pm} 0.3} & \textbf{45.3 \ensuremath{\pm} 0.3} & \textbf{40.9 \ensuremath{\pm} 0.2} & \textbf{30.8 \ensuremath{\pm} 0.2} & \textbf{53.5 \ensuremath{\pm} 0.4} & \textbf{39.7 \ensuremath{\pm} 0.3} \\
\bottomrule
\end{tabular}
\end{table}


%

The results show that BiRC2D is the only tested untrained extractor that consistently supports strong cold-start EFCIL. Random CNNs produce non-trivial but much weaker features, simple Patch-RNN variants show that sequentializing image patches is insufficient, and scattering features remain below BiRC2D. Thus, the results suggest that CIRCLE's gains come from spatially structured reservoir features, not merely from using any fixed random representation.

\subsubsection{Component-wise BiRC2D Adaptations}
\label{sec:appendix-components}

To understand why the adapted BiRC2D extractor performs well, we measure the contribution of our core modifications (see Figure~\ref{fig:circle-extractor}): the CNN stem, reservoir activation, weight initialization, and final random projection. We start from the main $T=20$ configuration on each dataset and change only the component being studied, keeping the remaining hyperparameters fixed. Unlike the previous ablations, these experiments use three random seeds.

\begin{table}[!htb]
\centering
\small
\caption{Ablations to our modification of BiRC2D architecture at $T=20$. Mean incremental accuracy $\bar{A}$ (\%), three seeds. The initialization experiments keep ReLU. The original configuration already uses Gaussian initialization on ImageNet-Subset and Kaiming-uniform on CIFAR-100.}
\label{tab:component-ablation}
\scriptsize
\begin{tabular}{lccc}
\toprule
Configuration & CIFAR-100 & TinyImageNet & ImageNet-Subset \\
\midrule
No CNN stem & 54.54 $\pm$ 0.23 & 39.64 $\pm$ 0.18 & 51.48 $\pm$ 0.32 \\
\midrule
Gaussian initialization & 56.04 $\pm$ 0.19 & 40.87 $\pm$ 0.17 & 53.49 $\pm$ 0.61 \\
Kaiming-uniform initialization & 55.97 $\pm$ 0.38 & 40.72 $\pm$ 0.13 & 53.11 $\pm$ 0.19 \\
\midrule
Tanh activation & 49.27 $\pm$ 0.32 & 36.73 $\pm$ 0.21 & 49.40 $\pm$ 0.21 \\
Tanh + Gaussian initialization & 49.03 $\pm$ 0.28 & 38.91 $\pm$ 0.08 & 49.40 $\pm$ 0.21 \\
\midrule
CIRCLE (original) & 55.97 $\pm$ 0.38 & 41.04 $\pm$ 0.25 & 53.49 $\pm$ 0.61 \\
\bottomrule
\end{tabular}
\end{table}

The results are reported in Table~\ref{tab:component-ablation}. Both the convolutional stem and the activation choice contribute to performance. Removing the convolutional stem lowers average incremental accuracy by 1.4-2.0 percentage points, while replacing ReLU with tanh lowers it by 4.1-6.7 points. Changing only the initialization has a smaller effect, at most 0.38 points. The gap to Conv+Patch-RNN in Table~\ref{tab:ablation-extractor} also supports using the spatial structure of images, although that comparison changes several components together.

\paragraph{Projection and head size.}\label{sec:appendix-projection}
The final random projection layer at the end of our adapted BiRC2D architecture (see Figure~\ref{fig:circle-extractor}) controls feature dimensionality and limits memory use in SLDA. Removing it substantially increases memory requirements, so we compare both variants using one prediction head and keep the feature-ensemble sizes unchanged.

\begin{table}[!htb]
\centering
\small
\caption{Final-projection ablation at $T=20$ with one prediction head in both variants: mean incremental accuracy $\bar{A}$ (\%), three seeds. All other settings, including feature-ensemble size, remain fixed.}
\label{tab:projection-ablation}
\scriptsize
\begin{tabular}{lccc}
\toprule
Configuration & CIFAR-100 & TinyImageNet & ImageNet-Subset \\
\midrule
With projection & 46.61 $\pm$ 0.22 & 35.49 $\pm$ 0.09 & 42.87 $\pm$ 1.08 \\
Without projection & 44.64 $\pm$ 0.11 & 35.99 $\pm$ 0.01 & 43.06 $\pm$ 0.89 \\
\bottomrule
\end{tabular}
\end{table}

The results in Table~\ref{tab:projection-ablation} show that removing the projection has only a small effect on accuracy. The random projection therefore offers a way to control feature dimensionality and substantially reduce memory requirements with little impact on accuracy.

\subsubsection{Classification Head}
\label{sec:exp-ablation-head}

We ablate the classification head while keeping the BiRC2D feature extractor fixed. We compare SLDA with several streaming analytic alternatives: Euclidean NCM, cosine NCM, ridge regression with recursive least squares (RLS), diagonal LDA, a FeCAM-style Mahalanobis classifier, and QDA with class-specific covariance estimates. These heads are described in the Appendix \ref{sec:appendix-headablation}. 

SLDA is consistently the strongest analytic head; full results are in Table~\ref{tab:ablation-head}. Prototype heads lose second-order information, diagonal LDA is too restrictive, and class-specific covariance methods such as FeCAM-Mahalanobis and QDA are less stable because per-class covariance estimation is data-hungry under cold start. Shared-covariance SLDA pools covariance across classes and gives the best balance between second-order modeling and statistical stability.

\begin{table}[ht]
\scriptsize
\centering
\caption{Classification head ablation at \(T=20\) on CIFAR-100, TinyImageNet, and ImageNet-Subset. All heads are paired with the same BiRC2D feature extractor and ensembling scheme. Mean \(\pm\) std over 5 seeds.}
\label{tab:ablation-head}
\begin{tabular}{lcccccc}
\toprule
 & \multicolumn{2}{c}{CIFAR-100} & \multicolumn{2}{c}{TinyImageNet} & \multicolumn{2}{c}{ImageNet-Subset} \\
\cmidrule(lr){2-3}\cmidrule(lr){4-5}\cmidrule(lr){6-7}
Head & \(\bar{A}\) & \(A_T\) & \(\bar{A}\) & \(A_T\) & \(\bar{A}\) & \(A_T\) \\
\midrule
Euclidean NCM & 23.56 \ensuremath{\pm} 0.2 & 14.96 \ensuremath{\pm} 0.1 & 14.09 \ensuremath{\pm} 0.1 & 7.91 \ensuremath{\pm} 0.0 & 25.29 \ensuremath{\pm} 0.7 & 13.83 \ensuremath{\pm} 0.4 \\
Cosine NCM & 23.48 \ensuremath{\pm} 0.2 & 14.90 \ensuremath{\pm} 0.1 & 14.06 \ensuremath{\pm} 0.0 & 7.90 \ensuremath{\pm} 0.0 & 25.26 \ensuremath{\pm} 0.7 & 13.81 \ensuremath{\pm} 0.4 \\
Ridge RLS & 50.97 \ensuremath{\pm} 0.1 & 37.76 \ensuremath{\pm} 0.2 & 13.60 \ensuremath{\pm} 7.4 & 5.49 \ensuremath{\pm} 3.3 & 48.55 \ensuremath{\pm} 0.1 & 32.96 \ensuremath{\pm} 0.4 \\
Diagonal LDA & 24.72 \ensuremath{\pm} 0.1 & 15.90 \ensuremath{\pm} 0.0 & 15.19 \ensuremath{\pm} 0.1 & 8.63 \ensuremath{\pm} 0.1 & 25.92 \ensuremath{\pm} 0.8 & 14.31 \ensuremath{\pm} 0.6 \\
FeCAM Mahalanobis & 49.36 \ensuremath{\pm} 0.1 & 39.40 \ensuremath{\pm} 0.1 & 27.41 \ensuremath{\pm} 0.9 & 19.34 \ensuremath{\pm} 0.7 & 21.58 \ensuremath{\pm} 2.2 & 11.18 \ensuremath{\pm} 1.4 \\
QDA (class-spec.\ cov) & 48.04 \ensuremath{\pm} 0.1 & 37.58 \ensuremath{\pm} 0.2 & 30.78 \ensuremath{\pm} 0.3 & 21.89 \ensuremath{\pm} 0.3 & 34.51 \ensuremath{\pm} 1.2 & 21.14 \ensuremath{\pm} 0.7 \\
\midrule
SLDA (CIRCLE) & \textbf{55.9 \ensuremath{\pm} 0.3} & \textbf{45.3 \ensuremath{\pm} 0.3} & \textbf{40.9 \ensuremath{\pm} 0.2} & \textbf{30.8 \ensuremath{\pm} 0.2} & \textbf{53.5 \ensuremath{\pm} 0.4} & \textbf{39.7 \ensuremath{\pm} 0.3} \\
\bottomrule
\end{tabular}
\end{table}

\subsubsection{Ensembling: Feature Dimension vs.\ Prediction Dimension}
\label{sec:exp-ablation-ensembling}
CIRCLE uses a fixed budget of \(n=km\) reservoir instances, partitioned into \(k\) groups of size \(m\). The group size \(m\) controls feature ensembling: the \(m\) reservoir embeddings within each group are concatenated and passed to one SLDA head. The number of groups \(k\) controls prediction ensembling: the \(k\) independently trained heads are averaged at the probability level. Thus, increasing \(m\) gives each head a richer representation, while increasing \(k\) increases prediction-level diversity. We ablate both axes on CIFAR-100 at \(T=20\) to verify that the two ensembling mechanisms are complementary and that balanced groupings perform best under a fixed reservoir budget.


\begin{figure}[ht]
\centering
\begin{minipage}[t]{0.35\linewidth}
 \vspace{0pt}
  \centering
  \includegraphics[width=\linewidth]{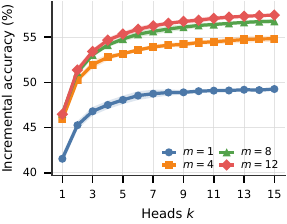}
\end{minipage}\hfill
\begin{minipage}[t]{0.35\linewidth}
\vspace{0pt}
  \centering
  \includegraphics[width=\linewidth]{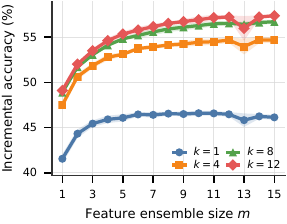}
\end{minipage}\hfill
\begin{minipage}[t]{0.29\linewidth}
\vspace{27pt}
  \centering
  \begingroup
  \small
  \setlength{\tabcolsep}{3.5pt}
  \scriptsize
  \begin{tabular}{ccc}
  \toprule
  $m$ & $k$ & \(\bar{A}\) (\%) \\
  \midrule
  4 & 4 & \textbf{52.97 \(\pm\) 0.31} \\
  1 & 16 & 49.35 \(\pm\) 0.11 \\
  16 & 1 & 45.37 \(\pm\) 0.99 \\
  \bottomrule
  \end{tabular}
  \endgroup
\end{minipage}
\caption{Ensembling ablation on CIFAR-100 at \(T=20\). Left: varying the number of prediction heads \(k\) for fixed feature-ensembling sizes \(m\). Middle: varying \(m\) for fixed \(k\).
The table compares three allocations with the same total reservoir budget \(mk=16\).
}
\label{fig:ablation-ensembling}
\end{figure}



Figure~\ref{fig:ablation-ensembling} shows that both ensembling axes contribute. For fixed feature-ensembling size \(m\), increasing the number of heads \(k\) improves incremental accuracy; for fixed \(k\), increasing \(m\) also improves performance. The two are complementary: feature ensembling enriches each head’s representation, while prediction ensembling reduces variance across independently drawn reservoir groups.

The fixed-budget comparison confirms the axes are not interchangeable. With the same total budget \(mk=16\), the balanced setting \(m=4,k=4\) outperforms both \(m=1,k=16\) and \(m=16,k=1\). Many small heads leave each classifier under-represented, while a single large feature vector removes the benefit of prediction averaging. A balanced allocation thus provides the best trade-off between representation quality and ensemble diversity, consistent with our theoretical results (Appendix~\ref{sec:appendix-theory}).

\textbf{Takeaway from ablations.} CIRCLE's superior performance is not explained by any single component such as freezing, random features, or ensembling. CIRCLE needs all three design choices: a spatially structured reservoir feature extractor, an SLDA head, and a balanced allocation between feature and prediction ensembling. Furthermore, higher feature dimensionality alone does not explain these gains: CIRCLE outperforms ACIL and DS-AL at long horizons despite their similarly large feature spaces, while increasing dimensionality gives only limited gains for EFC++ and AdaGauss (Appendix~\ref{sec:appendix-width}).


\section{Discussion and Limitations}
\label{sec:limits}

CIRCLE is specifically designed for long-horizon cold-start exemplar-free class-incremental learning. It is not intended as a universal continual learning method. In short-horizon settings, where only a small number of distribution shifts occur, trained-backbone drift-compensation methods such as EFC++ and AdaGauss outperform CIRCLE. CIRCLE's advantage appears when the task horizon is long enough that repeated representation updates and drift corrections become unreliable or expensive, and no sufficient data exists for strong initial representation learning. 

However, the method is designed primarily for low-cost streaming training, not necessarily minimal single-sample inference latency. At inference time, it evaluates multiple fixed reservoir instances and multiple SLDA heads; this cost is parallelizable and can be amortized by batching, but latency may be higher than single-backbone or single-head methods in small-batch or memory-constrained settings. The dominant inference cost is the reservoir ensemble. Reducing this cost through ensemble compression or distillation into a smaller student model is a promising future direction.


Our results show that not all data-free features are sufficient. Random CNNs, simple Patch-RNN reservoirs and scattering features do not match the adapted BiRC2D extractor. Thus, CIRCLE depends on a spatially structured reservoir architecture and careful tuning of reservoir dimensionality and ensemble allocation. Our experiments only cover natural-image classification under the cold-start exemplar-free protocol, so other modalities and warm-start settings remain untested. Extending CIRCLE to other modalities or higher-resolution domains may require similarly structured data-free extractors rather than direct reuse of the current architecture.

The theoretical analysis of grouped ensembling should be interpreted as motivation rather than a complete theory of BiRC2D reservoirs. Proposition~\ref{prop-biasvar} formalizes a bias--variance tradeoff under independence and approximation assumptions, and it explains why intermediate group sizes can be preferable to pure feature concatenation or pure prediction averaging. However, the approximation condition is stylized and is not a proof that the actual BiRC2D feature distribution satisfies a particular rate. The empirical ensemble ablation therefore remains essential evidence for the grouping design. We do not prove that finite random BiRC2D features remain discriminative as the number of classes or task complexity increases.

Finally, the ImageNet-1k $T=500$ experiment is a stress test rather than a fully exhaustive benchmark. We run two seeds because trained-backbone baselines are very expensive at this horizon, and some baselines are truncated or fail numerically before completing the full stream. We report these failures explicitly because they are part of the long-horizon behaviour under study, but the results should be interpreted together with the broader $T \in \{10,20,50,100\}$ evaluation across three datasets.

\clearpage
\begin{ack}

Augustinas Jučas and Yangchen Pan acknowledge the use of resources provided by the Isambard-AI National AI Research Resource (AIRR). Isambard-AI is operated by the University of Bristol and is funded by the UK Government's Department for Science, Innovation and Technology (DSIT) via UK Research and Innovation; and the Science and Technology Facilities Council [ST/AIRR/I-A-I/1023] \citep{mcintoshsmith2024isambardai}. Yangchen Pan acknowledges the support from the Engineering and Physical Sciences Research Council (EPSRC) New Investigator Award under grant reference UKRI2775.
The authors declare no competing interests.
\end{ack}
{
\small
\bibliography{references}
}

\newpage
\appendix
\onecolumn
The appendix includes the following contents. 

\begin{enumerate}
    \item Section \ref{sec:appendix-circle}: provides complete description of our proposed CIRCLE method.
    \item Section \ref{sec:appendix-theory}: Proposition \ref{prop-biasvar} characterizes the bias-variance trade-off of our two-level ensemble design introduced in Section \ref{sec:method:ensembling}, along with its detailed proof. 
    \item Section \ref{sec:appendix-experiment-details}: provides additional details and results on the experiments.
    \item Section \ref{sec:hyperparameters}: discusses and lists the hyperparameters for the main experiment.
\end{enumerate}

\section{CIRCLE: Additional Details}
\label{sec:appendix-circle}

Reservoir computing uses a fixed recurrent system to map signal inputs into a
high-dimensional feature space, while only the readout layer is trained \cite{jaeger2001echo, maass2002real}. Most reservoirs, including the canonical instantiation known as the Echo State Network (ESN) \cite{jaeger2001echo}, are suited to sequential data, but images are
two-dimensional objects. A simplistic conversion of an image into a one-dimensional pixel
stream to be passed into an ESN destroys the local spatial structure, which is why simple flattened
reservoir classifiers \cite{LopezOrtiz2024} are weak on non-trivial image benchmarks, as we show in an ablation later in this study. Therefore, more sophisticated reservoir architectures are required for image data. However, the literature on such models is very limited. To the best of our knowledge, the most suitable and relevant reservoir architecture for our work is BiRC2D, introduced by \citet{nakanishi2024birc2d}. 

\paragraph{BiRC2D description.} We therefore base our feature extractor on BiRC2D~\cite{nakanishi2024birc2d} -- a
bidirectional two-dimensional reservoir architecture designed for image data. The
model first divides the input image into patches of a fixed size and flattens the
pixels within each patch into a feature vector, producing a feature map whose
\textit{spatial locations} correspond to patches rather than to raw pixels. This
feature map is then passed through a stack of BiRC2D \textit{layers}, each of which
is a map $\mathbb{R}^{H \times W \times C_{\mathrm{in}}} \to \mathbb{R}^{H \times W \times C_{\mathrm{out}}}$
that gives every spatial location access to contextual information from multiple
directions. In particular, a single layer reshapes its input into sequences along
four directions -- left-to-right, right-to-left, top-to-bottom, and bottom-to-top --
runs each sequence through a fixed ESN reservoir, and concatenates the four directional
outputs along the channel dimension to recover a spatial map. To improve spatial
awareness, the full model runs several such stacks in parallel, each operating on a
different patch size, and concatenates their outputs at the end.

\paragraph{Our additions.} The original BiRC2D model was proposed for image anomaly detection, with the output of the model keeping the spatial $H \times W$ feature map structure, required by the context of the original study. However, our task is
different: we require a single image-level embedding \textit{vector} (not a 2D grid) for class-incremental
classification. Furthermore, our study concerns general images, as opposed to analyzing only spectral imaging in \cite{nakanishi2024birc2d}. Based on these differences, as well as empirical experimentation, we adapt the architecture in four ways. 
First, before applying the reservoir blocks, we pass the image through a stack of small and \textit{random} convolutional layers, inspired by \citet{8545471,chang2019reinforcementlearningconvolutionalreservoir}. This gives the reservoir local low-level features rather
than raw pixels. This is useful for RGB images, which have fewer input channels than the hyperspectral images used in the original BiRC2D. Second, we use ReLU activations inside the reservoir blocks instead
of hyperbolic tangent activations, as they perform better in our experiments (Section~\ref{sec:appendix-components}). Third, we initialize the fixed reservoir weights
using Kaiming uniform initialization on CIFAR-100, Xavier-uniform input weights with orthogonal recurrent weights on TinyImageNet, and Gaussian initialization on ImageNet-Subset.

Finally, after the final reservoir block, we compute the mean, maximum, and standard deviation across spatial locations for each patch size, and concatenate the results into a single vector. We then apply a fixed random projection, inspired by RanPAC~\cite{mcdonnell2024ranpacrandomprojectionspretrained} and ACIL~\cite{zhuang2022acilanalyticclassincrementallearning}. Our projection reduces the feature dimension to limit the size of SLDA's covariance matrix (Section~\ref{sec:appendix-projection}).

Figure~\ref{fig:circle-extractor} shows the complete extractor and the changes to the original BiRC2D.

\begin{figure}[!htb]
\centering
\includegraphics[width=0.8\linewidth,trim=0 0 44bp 0,clip]{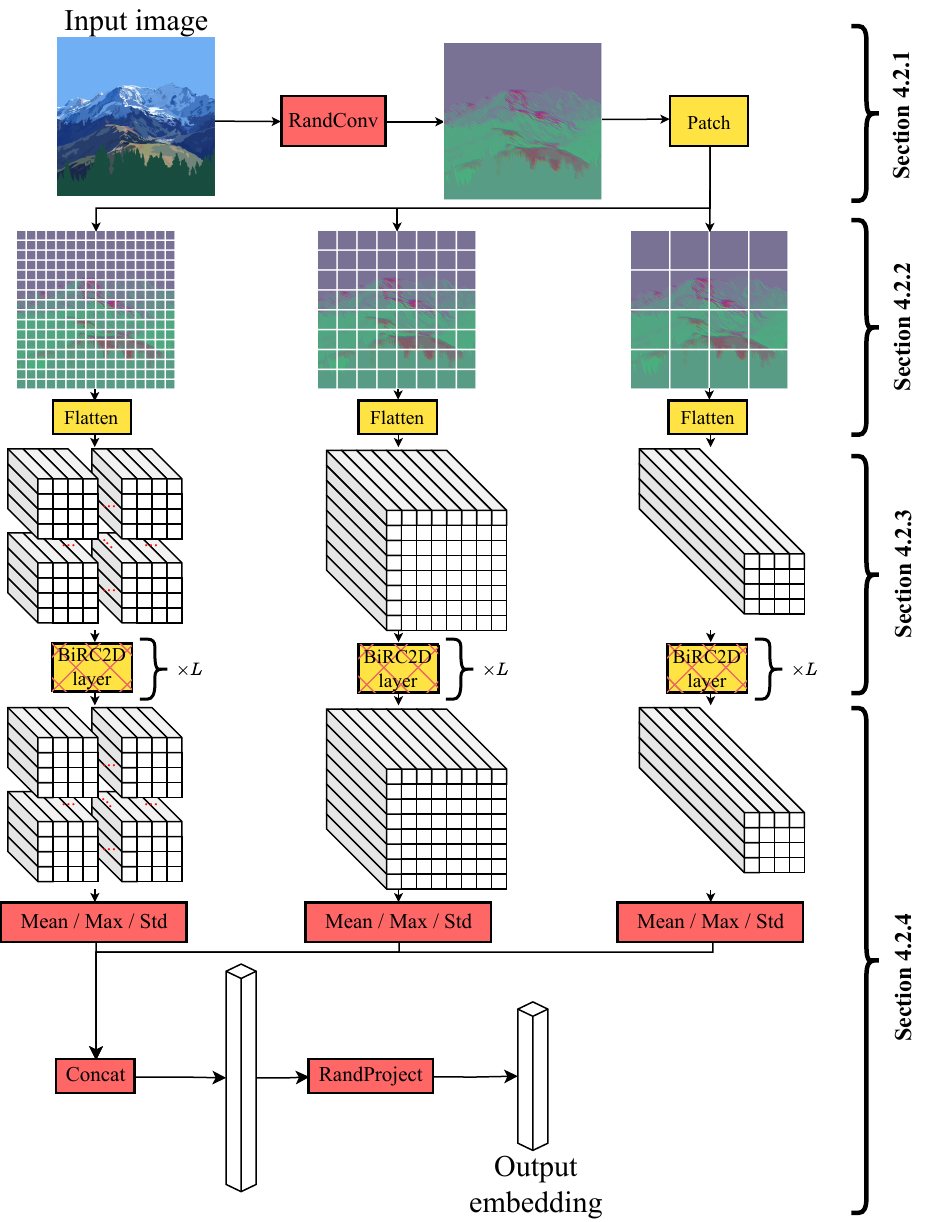}
\caption{CIRCLE's feature extractor, adapted from BiRC2D~\cite{nakanishi2024birc2d}. The multiscale, four-directional reservoir blocks are from the original BiRC2D. Red boxes mark added stages, and crosses mark changes to nonlinearity and initialization.}
\label{fig:circle-extractor}
\end{figure}

\paragraph{Multiple seeds.} We denote $\phi_s : \mathcal{X} \rightarrow \mathbb{R}^{d_r}$
to be the reservoir feature extractor instantiated with random seed $s$. The
seed determines all fixed random weights in the convolutional stack, reservoir blocks,
and final projection layer. No parameter of $\phi_s$ is trained.

Figure~\ref{fig:circle-ensembling} illustrates the full pipeline: features are concatenated within each reservoir group, and the softmax outputs of the independent SLDA heads are averaged.

\begin{figure}[!htb]
\centering
\includegraphics[width=0.75\linewidth]{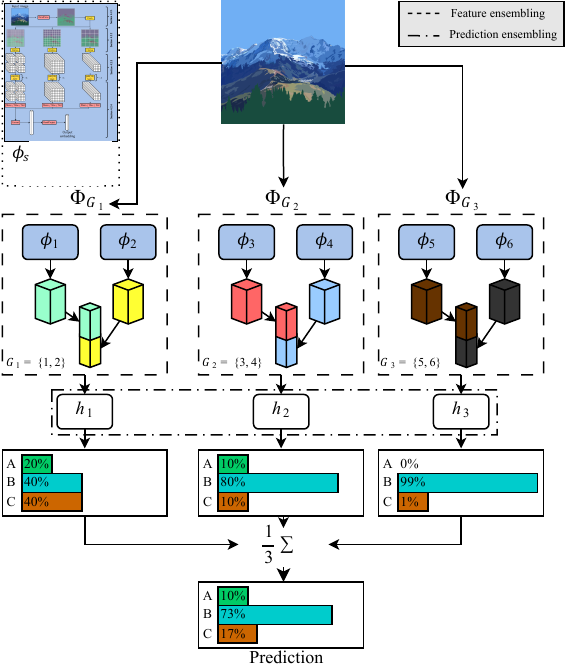}
\caption{High-level overview of CIRCLE, focusing on feature and prediction ensembling. The example uses six fixed reservoir extractors, $k=3$ groups, and $m=2$ extractors per group. Each $\phi_s$ is a modified BiRC2D-based feature extractor, depicted in Figure~\ref{fig:circle-extractor}. The concatenated group embedding $\Phi_{G_i}$ is passed to its SLDA head $h_i$, and the heads' softmax outputs are averaged to give the final prediction. The example probabilities are rounded.}
\label{fig:circle-ensembling}
\end{figure}

\paragraph{Our model in detail.}
Algorithms ~\ref{alg:circle-train}, ~\ref{alg:circle-predict} and ~\ref{alg:circle-materialise} give the training and prediction procedures of CIRCLE. Training takes a stream $\mathcal{S} = \{(x_t, y_t)\}_{t=1}^{N}$ presented in any order. Given $n$ reservoir seeds, CIRCLE partitions them into $k$ disjoint groups $\mathcal{G}=\{G_i\}_{i=1}^{k}$, with $|G_i|=m=n/k$. For each head $i$, $N_{i,c}$ counts the observed samples of class $c$, $T_{i,c}$ sums their embeddings, and $M_i$ stores the second moment across all classes. The algorithms are written for single samples; in practice, feature extraction and accumulator updates are vectorised over minibatches.

\begin{algorithm}[H]
\caption{CIRCLE -- training}
\label{alg:circle-train}
\begin{algorithmic}[1]
\Require Stream $\mathcal{S} = \{(x_t, y_t)\}_{t=1}^{N}$
\Require Reservoir groups $\mathcal{G}=\{G_i\}_{i=1}^{k}$, with $|G_i|=m$
\Statex
\For{$i = 1,\ldots,k$}
    \State Initialise head $h_i$
    \Statex \quad $N_{i,c} \leftarrow 0$, $T_{i,c} \leftarrow \mathbf{0}$ for all classes $c$
    \Statex \quad $M_i \leftarrow \mathbf{0}$
\EndFor
\Statex
\For{$(x,y) \in \mathcal{S}$}
    \For{$i = 1,\ldots,k$}
        \State $z_i \leftarrow \Phi_{G_i}(x)$ \Comment{$\Phi_{G_i}(x)=\Vert_{s\in G_i}\phi_s(x)$}
        \State $N_{i,y} \leftarrow N_{i,y}+1$
        \State $T_{i,y} \leftarrow T_{i,y}+z_i$
        \State $M_i \leftarrow M_i+z_i z_i^\top$
    \EndFor
\EndFor
\end{algorithmic}
\end{algorithm}

\begin{algorithm}[H]
\caption{CIRCLE -- prediction}
\label{alg:circle-predict}
\begin{algorithmic}[1]
\Require Test sample $x$
\Require Reservoir groups $\mathcal{G}=\{G_i\}_{i=1}^{k}$
\Require Heads $\{h_i\}_{i=1}^{k}$ with accumulators $\{N_{i,\cdot},T_{i,\cdot},M_i\}_{i=1}^{k}$
\Statex
\For{$i = 1,\ldots,k$}
    \State Materialise $\{\mu_{i,c}\}$, $\Sigma_{i,\lambda}$, and $\{\pi_{i,c}\}$ from $(N_{i,\cdot},T_{i,\cdot},M_i)$ \Comment{Use Alg. \ref{alg:circle-materialise}}
    \State $z_i \leftarrow \Phi_{G_i}(x)$ \Comment{$\Phi_{G_i}(x)=\Vert_{s\in G_i}\phi_s(x)$}
    \For{each observed class $c$}
        \State $\ell_{i,c} \leftarrow [h_i(z_i)]_c$ 
        \Comment{$\ell_{i,c} \leftarrow z_i^\top \Sigma_{i,\lambda}^{-1}\mu_{i,c}
    -\frac{1}{2}\mu_{i,c}^\top \Sigma_{i,\lambda}^{-1}\mu_{i,c}
    +\log \pi_{i,c}$}
    \EndFor
    \State $p_i \leftarrow \mathrm{softmax}(\ell_i)$
\EndFor
\State \Return $\hat{p}(x)=\frac{1}{k}\sum_{i=1}^{k}p_i$
\end{algorithmic}
\end{algorithm}

\begin{algorithm}[H]
\caption{\textsc{Materialise}: obtaining closed-form LDA classifier from accumulated statistics}
\label{alg:circle-materialise}
\begin{algorithmic}[1]
\Require Head accumulators $(n_{c}, T_{c}, M)$ for $c$ in observed classes
\State $N \leftarrow \sum_c n_c$;\quad $C \leftarrow |\{c : n_c > 0\}|$
\State $\mu_c \leftarrow T_c / n_c$ \Comment{class means}
\State $\pi_c \leftarrow n_c / N$ \Comment{class priors}
\State $S \leftarrow M - \sum_c n_c\,\mu_c \mu_c^\top$ \Comment{pooled within-class scatter}
\State $\Sigma_\lambda \leftarrow S/(N-C) + \lambda I$ \Comment{regularised covariance}
\State \Return $\{\mu_c\}, \Sigma_\lambda, \{\pi_c\}$
\end{algorithmic}
\end{algorithm}

\section{Proof for Proposition \ref{prop-biasvar}}
\textbf{Definitions and notations}. Let $x$ be some arbitrarily fixed input and $\eta(x)\in\Delta^{C}$ denote the target conditional label distribution over $C$ classes, where
$\Delta^{C}=\{p\in\mathbb{R}^{C}:p_c\geq 0,\ \sum_{c=1}^{C}p_c=1\}$.
Let $n=km$ be the total number of reservoir instantiations, partitioned into
$k$ disjoint groups $G_1,\ldots,G_k$, each of size $m$. For each group $G_j$, define
$z_j(x) \defeq \big\Vert_{s\in G_j}\phi_s(x)$, where $\phi_s$ is the fixed random reservoir feature extractor instantiated with seed $s$, and $\Vert$ denotes feature concatenation. Let $h_j$ be the classification head trained on
the features from group $G_j$, and define the group-level predicted probability vector
$p_j^{(m)}(x) \defeq \operatorname{softmax}\!\left(h_j(z_j(x))\right)\in \Delta^{C}$.
The grouped ensemble prediction is $\bar p_{k,m}(x)\defeq\frac{1}{k}\sum_{j=1}^{k}p_j^{(m)}(x)$. Define mean prediction $\mu_m(x)=\mathbb{E}\!\left[p_j^{(m)}(x)\right]$, where the expectation is over the random reservoir initializations. 


\begin{proposition}[Bias-variance decomposition for grouped reservoir ensembles]\label{prop-biasvar}
Assume that, conditional on the training data, the random vectors
$p_1^{(m)}(x),\ldots,p_k^{(m)}(x)$ are independent and identically distributed with respect to the reservoir initialization randomness. Define the squared bias $B_m(x)\defeq \left\|\mu_m(x)-\eta(x)\right\|_2^2$. 
Then
\begin{equation}\label{eq:ensemble-mse-decompose}
\mathbb{E}\!\left[\left\|\bar p_{k,m}(x)-\eta(x)\right\|_2^2\right] \leq B_m(x) + \frac{m}{n}.
\end{equation}

Furthermore, if the group-level squared bias obeys the approximation condition
\begin{equation}\label{eq:approximation-condition}
    B_m(x)\leq A m^{-2\alpha}
\end{equation}
for some constants $A>0$ and $\alpha>0$, then
\[
    \mathbb{E}\!\left[
        \left\|\bar p_{k,m}(x)-\eta(x)\right\|_2^2
    \right]
    \leq
    A m^{-2\alpha}
    +
    \frac{m}{n}.
\]
The right-hand side is minimized over positive real $m$ at $m^\star=\left(2\alpha A n\right)^{\frac{1}{2\alpha+1}}$.

In CIRCLE, $m$ must divide $n$ and satisfy $1\leq m\leq n$, since every group contains the same integer number of reservoirs. Therefore, this assumption can favour an intermediate group size, but the best valid choice may also be $m=1$ or $m=n$.
\end{proposition}

\begin{proof}
Fix $x$ throughout the proof and suppress the explicit dependence on $x$ to simplify
notation. Write
\[
    p_j = p_j^{(m)}(x),
    \qquad
    \bar p = \bar p_{k,m}(x),
    \qquad
    \mu_m = \mathbb{E}[p_j],
    \qquad
    \eta = \eta(x).
\]
By definition,
\[
    \bar p
    =
    \frac{1}{k}\sum_{j=1}^{k}p_j.
\]
We decompose the prediction error as
\[
    \bar p-\eta
    =
    (\bar p-\mu_m)+(\mu_m-\eta).
\]
Taking squared Euclidean norm gives
\[
    \|\bar p-\eta\|_2^2
    =
    \|\bar p-\mu_m\|_2^2
    +
    \|\mu_m-\eta\|_2^2
    +
    2\langle \bar p-\mu_m,\mu_m-\eta\rangle.
\]
Taking expectation over the reservoir randomness,
\[
    \mathbb{E}\|\bar p-\eta\|_2^2
    =
    \mathbb{E}\|\bar p-\mu_m\|_2^2
    +
    \|\mu_m-\eta\|_2^2
    +
    2\left\langle
        \mathbb{E}[\bar p-\mu_m],
        \mu_m-\eta
    \right\rangle.
\]
Now
\[
    \mathbb{E}[\bar p]
    =
    \mathbb{E}\left[\frac{1}{k}\sum_{j=1}^{k}p_j\right]
    =
    \frac{1}{k}\sum_{j=1}^{k}\mathbb{E}[p_j]
    =
    \mu_m,
\]
so
\[
    \mathbb{E}[\bar p-\mu_m]=0.
\]
Therefore the cross term vanishes, and
\[
    \mathbb{E}\|\bar p-\eta\|_2^2
    =
    \mathbb{E}\|\bar p-\mu_m\|_2^2
    +
    \|\mu_m-\eta\|_2^2.
\]
By definition,
\[
    \|\mu_m-\eta\|_2^2
    =
    B_m.
\]

It remains to compute the variance term. Define
\[
    \xi_j = p_j-\mu_m.
\]
Then
\[
    \mathbb{E}[\xi_j]
    =
    \mathbb{E}[p_j-\mu_m]
    =
    \mathbb{E}[p_j]-\mu_m
    =
    0,
\]
and
\[
    \bar p-\mu_m
    =
    \frac{1}{k}\sum_{j=1}^{k}\xi_j.
\]
Therefore
\[
    \mathbb{E}\|\bar p-\mu_m\|_2^2
    =
    \mathbb{E}\left\|
        \frac{1}{k}\sum_{j=1}^{k}\xi_j
    \right\|_2^2.
\]
Expanding the squared norm,
\[
    \mathbb{E}\|\bar p-\mu_m\|_2^2
    =
    \frac{1}{k^2}
    \mathbb{E}\left[
        \left\langle
            \sum_{i=1}^{k}\xi_i,
            \sum_{j=1}^{k}\xi_j
        \right\rangle
    \right]
    =
    \frac{1}{k^2}
    \sum_{i=1}^{k}\sum_{j=1}^{k}
    \mathbb{E}\left[\langle \xi_i,\xi_j\rangle\right].
\]
For $i=j$,
\[
    \mathbb{E}\left[\langle \xi_i,\xi_i\rangle\right]
    =
    \mathbb{E}\|\xi_i\|_2^2
    =
    \mathbb{E}\|p_i-\mu_m\|_2^2
    \defeq
    \sigma_m^2,
\]
which is the single-group prediction variance at input $x$. 
For $i\neq j$, independence of $p_i$ and $p_j$ implies independence of $\xi_i$ and
$\xi_j$. Hence
\[
    \mathbb{E}\left[\langle \xi_i,\xi_j\rangle\right]
    =
    \mathbb{E}\left[\sum_{c=1}^{C}\xi_{i,c}\xi_{j,c}\right]
    =
    \sum_{c=1}^{C}
    \mathbb{E}[\xi_{i,c}\xi_{j,c}].
\]
By independence,
\[
    \mathbb{E}[\xi_{i,c}\xi_{j,c}]
    =
    \mathbb{E}[\xi_{i,c}]\mathbb{E}[\xi_{j,c}]
    =
    0\cdot 0
    =
    0.
\]
Therefore,
\[
    \mathbb{E}\left[\langle \xi_i,\xi_j\rangle\right]=0
    \qquad
    \text{for all } i\neq j.
\]
Thus
\[
    \mathbb{E}\|\bar p-\mu_m\|_2^2
    =
    \frac{1}{k^2}
    \sum_{j=1}^{k}\sigma_m^2
    =
    \frac{\sigma_m^2}{k}.
\]
Combining the bias and variance terms yields
\[
    \mathbb{E}\|\bar p-\eta\|_2^2
    =
    B_m
    +
    \frac{\sigma_m^2}{k}.
\]

We now prove the bound $\sigma_m^2\leq 1$. Since $p_j\in\Delta^C$, every coordinate of
$p_j$ is nonnegative and $\sum_{c=1}^{C}p_{j,c}=1$. Therefore
\[
    \|p_j\|_2^2
    =
    \sum_{c=1}^{C}p_{j,c}^2
    \leq
    \left(\sum_{c=1}^{C}p_{j,c}\right)^2
    =
    1.
\]
Also,
\[
    \sigma_m^2
    =
    \mathbb{E}\|p_j-\mu_m\|_2^2.
\]
Expanding this expression,
\[
    \sigma_m^2
    =
    \mathbb{E}\left[
        \|p_j\|_2^2
        -2\langle p_j,\mu_m\rangle
        +\|\mu_m\|_2^2
    \right].
\]
Because $\mu_m=\mathbb{E}[p_j]$,
\[
    \mathbb{E}\langle p_j,\mu_m\rangle
    =
    \left\langle \mathbb{E}[p_j],\mu_m\right\rangle
    =
    \langle \mu_m,\mu_m\rangle
    =
    \|\mu_m\|_2^2.
\]
Hence
\[
    \sigma_m^2
    =
    \mathbb{E}\|p_j\|_2^2
    -
    \|\mu_m\|_2^2.
\]
Since $\mathbb{E}\|p_j\|_2^2\leq 1$ and $\|\mu_m\|_2^2\geq 0$,
\[
    \sigma_m^2\leq 1.
\]
Also $\sigma_m^2\geq 0$ because it is the expectation of a squared norm. Therefore
\[
    0\leq \sigma_m^2\leq 1.
\]
Substituting this into the exact decomposition gives
\[
    \mathbb{E}\|\bar p-\eta\|_2^2
    =
    B_m+\frac{\sigma_m^2}{k}
    \leq
    B_m+\frac{1}{k}.
\]
Since $n=km$, we have $k=n/m$, and therefore
\[
    \frac{1}{k}
    =
    \frac{m}{n}.
\]
Thus
\[
    \mathbb{E}\|\bar p_{k,m}(x)-\eta(x)\|_2^2
    \leq
    B_m(x)+\frac{m}{n}.
\]

Finally, suppose $B_m(x)\leq A m^{-2\alpha}$ for some $A>0$ and $\alpha>0$. Then
\[
    \mathbb{E}\|\bar p_{k,m}(x)-\eta(x)\|_2^2
    \leq
    A m^{-2\alpha}
    +
    \frac{m}{n}.
\]
Define
\[
    R(m)
    =
    A m^{-2\alpha}
    +
    \frac{m}{n}
\]
for $m>0$. Its derivative is
\[
    R'(m)
    =
    -2\alpha A m^{-2\alpha-1}
    +
    \frac{1}{n}.
\]
Setting $R'(m)=0$ gives
\[
    2\alpha A m^{-2\alpha-1}
    =
    \frac{1}{n}.
\]
Equivalently,
\[
    m^{2\alpha+1}
    =
    2\alpha A n.
\]
Thus the unique stationary point is
\[
    m^\star
    =
    \left(2\alpha A n\right)^{\frac{1}{2\alpha+1}}.
\]
Moreover,
\[
    R''(m)
    =
    2\alpha(2\alpha+1)A m^{-2\alpha-2}
    >
    0
\]
for all $m>0$. Therefore $m^\star$ is the unique minimizer of $R(m)$ over positive real
$m$.
\end{proof}

\begin{remark}
The decomposition holds for any $B_m(x)$. The bound $B_m(x)\le A m^{-2\alpha}$ is a stylized model of how bias may decrease as more independent random reservoirs are concatenated. Such polynomial rates are standard in approximation theory and align with classical random-feature results, e.g., $O(1/m)$ rates for shallow networks~\citep{barron1993universal} and convergence bounds for random features in kernel methods~\citep{rahimi2007random,rahimi2008weighted,rudi2017generalization}. We use $A m^{-2\alpha}$ only as a phenomenological model, not a theorem for BiRC2D.
\end{remark}

\paragraph{Interpretation.} The result explains the need for both feature concatenation and prediction averaging. Increasing $m$ enriches each head’s representation (captured by $B_m(x)$), but reduces the number of heads to $k=n/m$, so prediction averaging contributes at most $1/k=m/n$ variance. Thus, the grouped ensemble trades bias $B_m(x)$ against variance $\le m/n$. Under $B_m(x)\le A m^{-2\alpha}$, the bound becomes $A m^{-2\alpha}+\frac{m}{n}$: the first term decreases with $m$, the second increases. Hence, $m=1$ (prediction-only) leaves high bias, $m=n$ (feature-only) removes variance reduction, and an intermediate $m$ balances both, explaining the benefit of a balanced ensemble.\label{sec:appendix-theory}

\section{Additional Experimental Details}
\label{sec:appendix-experiment-details}

This section provides additional details and results behind the performed experiments.  

All CIRCLE ablation baselines are designed to preserve the main online-learning setting of CIRCLE. In particular, the feature-extractor ablations replace only the frozen image representation, while keeping the same analytic SLDA head and the same ensembling mechanism. Conversely, the head ablations replace only the final analytic classifier, while keeping the BiRC2D feature extractor fixed. All variants are updated in a streaming manner and are trained without gradient-based optimisation of the feature extractor or classification head. The baseline feature-dimension experiments in Section~\ref{sec:appendix-width} use each method's training procedure.

The code for all experiments can be found at  \url{http://github.com/augustinasjucas/circle}.

\subsection{Detailed Experimental Setup}\label{sec:appendix-detailedsetup}

\textbf{Datasets and protocol.}
We evaluate on three standard CIL benchmarks: CIFAR-100 (100 classes, $32{\times}32$)~\cite{krizhevsky2009learning}, TinyImageNet (200 classes, $64{\times}64$)~\cite{wu2017tiny}, and ImageNet-Subset (100 classes from ImageNet-1k at $224{\times}224$)~\cite{ILSVRC15}. For long-horizon experiments, we additionally use the full ImageNet-1k. All experiments are conducted in the cold-start exemplar-free regime: classes are split evenly across $T$ tasks, no replay buffer is used, and no pre-trained weights or external data are allowed.

For the main comparison, we consider $T \in \{10, 20, 50, 100\}$. The settings $T=10$ and $T=20$ are standard in the EFCIL literature~\citep{magistri2024elasticfeatureconsolidationcold,grzegorz2024taskrecency}. The longer horizon $T=50$ has appeared occasionally in warm-start studies (e.g., DS-AL~\citep{zhuang2024dsaldualstreamanalyticlearning}) but, to our knowledge, not in cold-start EFCIL on standard benchmarks. The $T=100$ setting is likewise unexplored in cold-start CIL; for CIFAR-100 and ImageNet-Subset, it corresponds to one class per task, i.e., an extreme distribution-shift regime. For full ImageNet-1k, we further evaluate $T=500$ (two classes per task), following~\citet{zhuang2024dsaldualstreamanalyticlearning}. Results are averaged over 5 seeds, except for $T=500$, where we use 2 seeds due to the high cost of training-based baselines.

Our protocol for splitting a dataset into multiple tasks is as follows. Assume we wish to split a dataset into $T$ tasks, where each task has  $C$ classes. We fix a \textit{data} seed (we use seed $0$) using which we shuffle the classes. Then we assign the samples from the first $C$ classes from the list into $\mathcal{T}_1$, the next $C$ classes into $\mathcal{T}_2$, etc. For different splits (i.e., different $T$ values), we still use the same shuffling seed $0$, this way the splits are consistent. Across multiple random instantiations of the same experiment we \textbf{do not} change the \textit{data} seed, i.e., the underlying data split remains the same across re-runs of the experiment. The varied seeds vary mostly model parameters, so model weight initializations, but also some data parameters, namely the augmentations for models that use them.

\textbf{Baselines.}
We compare against representative EFCIL methods across key design axes. 
\emph{Drift-compensation / distillation}: EFC++~\citep{magistri2025efcelasticfeatureconsolidation}, AdaGauss~\citep{grzegorz2024taskrecency}, ADC~\citep{goswami2024resurrectingoldclassesnew}, LwF~\cite{DBLP:journals/corr/LiH16e}. 
\emph{Prototype rehearsal / augmentation}: PASS~\citep{Zhu_2021_CVPR}, IL2A~\citep{zhu2021class}. 
\emph{Frozen-backbone analytic methods} (adapted to cold-start by training on $\mathcal{T}_1$ and freezing thereafter): ACIL~\citep{zhuang2022acilanalyticclassincrementallearning}, DS-AL~\citep{zhuang2024dsaldualstreamanalyticlearning}, FeCAM~\citep{goswami2024fecamexploitingheterogeneityclass}, FeTrIL~\citep{petit2023fetrilfeaturetranslationexemplarfree}. 
All trained-backbone baselines use ResNet-18, following standard EFCIL practice. Implementations for ACIL, DS-AL, FeTrIL, IL2A, LwF, and PASS are from PyCIL~\cite{zhou2023pycil}; others use official code.

\textbf{Hyperparameter tuning.}
We tune all baselines for each $(T,\text{dataset})$ configuration. This is crucial at $T=50$ and $T=100$, where prior hyperparameters perform poorly. Our protocol thus provides a fair assessment at long horizons. Details and optimal values for all methods, including CIRCLE, are given in Appendix~\ref{sec:hyperparameters}.


\textbf{Metrics.} We report \emph{average incremental accuracy} ($\bar{A}$) and \emph{final accuracy} ($A_T$). For efficiency, we also report total wall-clock training time, measured by running each experiment in isolation on a single NVIDIA L40S GPU.

\textbf{Hardware and compute.} The experiments in Table ~\ref{tab:main-comparison} were tuned and executed on NVIDIA A100 GPUs (40 GB memory per GPU) in a SLURM cluster, with the experiment lengths having strongly varied depending on the model that was used. However, no single experiment took more than 24 hours to finish. The hyperparameter tuning for the ablation experiments were performed by utilizing eight B200 GPUs (180 GB memory per GPU), used for less than 48 hours in total. Finally, the duration experiments were performed on 8 NVIDIA L40S GPUs (48 GB memory per GPU), each experiment taking exactly as long as specified in Table~\ref{tab:time}. The additional component and feature-dimension experiments use NVIDIA RTX A5500 GPUs with 24 GB memory.

\textbf{Licensing of the data and code.} CIFAR-100 does not include a specified license, ImageNet-1k, ImageNet-Subset and TinyImageNet are used under the ImageNet terms of access, which restrict use to non-commercial research and educational purposes. To obtain baseline results, we incorporate the codebases of the baselines into our project, which is allowed under MIT license which is used by the codebases of EFC++~\cite{magistri2025efcelasticfeatureconsolidation}, ADC~\cite{goswami2024resurrectingoldclassesnew}, FeCAM~\cite{goswami2024fecamexploitingheterogeneityclass} and PyCIL (PyCIL covers LwF~\cite{DBLP:journals/corr/LiH16e}, PASS~\cite{Zhu_2021_CVPR}, IL2A~\cite{zhu2021class}, ACIL~\cite{zhuang2022acilanalyticclassincrementallearning}, DS-AL~\cite{zhuang2024dsaldualstreamanalyticlearning} and FeTRiL~\cite{petit2023fetrilfeaturetranslationexemplarfree}). The codebase of AdaGauss is the only one that provides no license, with the \textit{arxiv} version of the paper being licensed under CC BY 4.0.

\paragraph{One class per task.}\label{sec:appendix-single-class}
At $T=100$ on CIFAR-100 and ImageNet-Subset, a task-local head has one logit, so its softmax is one and its cross-entropy loss and gradient are zero. We retain the official AdaGauss and EFC++ objectives and heads. Both remain executable, but their later-task backbone updates lose the supervised class-discrimination signal from this term.

AdaGauss~\cite{grzegorz2024taskrecency} uses an auxiliary four-way rotation objective on the first task. On later tasks, projected feature distillation and anti-collapse terms remain active. Its per-task softmax head is discarded; inference uses the class-conditional Gaussian classifier, whose covariance estimation and adaptation remain defined with one class in the current task.

EFC++~\cite{magistri2025efcelasticfeatureconsolidation} also uses an auxiliary rotation objective on the first task. Later backbone updates retain the Empirical Feature Matrix regularizer with frozen batch normalization. During prototype rebalancing, the backbone is frozen and all classification heads are jointly trained using a single softmax over their concatenated outputs, on current-task features and sampled old-class features. This supplies classifier fitting and cross-head calibration despite the degenerate task-local backbone loss.

Thus, the one-class setting exposes a limitation of task-local supervised representation learning as well as the difficulties of drift compensation. It should not be interpreted as isolating drift-estimation error alone. The $T=50$ results and ImageNet-1k $T=500$ stress test also include settings with multiple classes per task.

\subsection{Interpreting Task Partition Results}

\label{sec:appendix-evaluation-grid}
In Table~\ref{tab:main-comparison}, CIRCLE's $A_T$ remains constant across task splits, while $\bar A_T$ gradually increases with $T$, the following is an explanation for such behaviour.
Fix the ordered class list, feature extractor, and regularization, and let $a(j)$ be the accuracy after observing the first $j$ classes. Given $C_{\mathrm{total}}$ classes split evenly into $T$ tasks, the two reported metrics are:
\[
\bar{A}_T=\frac{1}{T}\sum_{t=1}^{T}a\!\left(\frac{tC_{\mathrm{total}}}{T}\right),
\qquad A_T=a(C_{\mathrm{total}}).
\]
Since SLDA's statistics are additive, they produce the same classifier given the same observed-data prefix, regardless of how that prefix was partitioned into tasks. Hence the $A_T$ values are approximately equal along different $T$ values. However, they do not make averages over different evaluation grids equal. For CIFAR-100, the $T=10$ split averages accuracy after 10, 20, \ldots, 100 classes, while $T=100$ averages after every class. The additional early prefixes are generally easier, which explains the observed increase in $\bar{A}$.

\subsection{Additional Results on Performance vs Wall-clock Time}\label{sec:appendix-wallclocktimefig}

Figure~\ref{fig:duration-accuracy-tradeoff} provides an additional view of the accuracy--time trade-off. For any fixed wall-clock budget, it shows how far each method can progress through the incremental stream and what accuracy it achieves at that point. In most settings, CIRCLE completes the entire \(T=20\) stream before most baselines have finished training even on the first task. This is particularly important when comparing against freeze-after-first-task methods: although these methods avoid later backbone updates, they still require an expensive first-task training stage before incremental learning can begin. CIRCLE has no such warm-up stage and can update incrementally from the very first sample obtained. 

\begin{figure}[ht]
\centering
\includegraphics[width=0.95\linewidth]{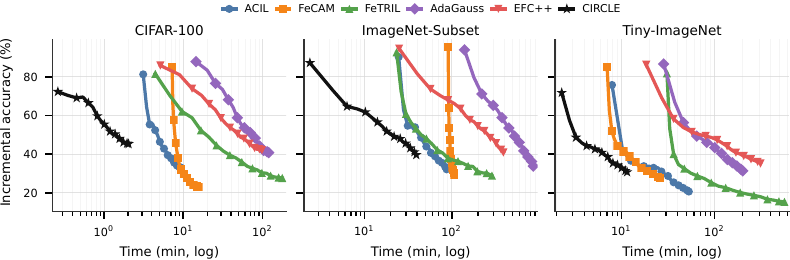}
\caption{Accuracy-computation time trade-off at \(T=20\). Each panel plots incremental test accuracy up to the current task against average wall-clock time of training up to the current task.}
\label{fig:duration-accuracy-tradeoff}
\end{figure}

\subsection{Feature-extractor Ablations}

Table~\ref{tab:app-feature-extractor-baselines} summarises the feature extractors used in the ablation study. The feature-extractor ablations test whether CIRCLE's performance comes from the BiRC2D representation itself or merely from using a frozen high-dimensional image embedding. The random ResNet-18, ResNet-18-LN, and VGG-13 baselines compare against standard convolutional architectures without training. Patch-RNN and Conv+Patch-RNN test simpler reservoir-style alternatives that sequentialise the image, with and without local random convolutional preprocessing. The scattering transform provides a deterministic, theory-driven non-trained baseline.

\begin{table}[ht]
\centering
\small
\caption{Feature-extractor baselines used in the ablation study. All feature extractors are frozen and are followed by the same SLDA classification head.}
\label{tab:app-feature-extractor-baselines}
\begin{tabular}{p{0.1\linewidth}p{0.38\linewidth}p{0.43\linewidth}}
\toprule
\textbf{Extractor} & \textbf{Description} &\textbf{ Purpose of the ablation} \\
\midrule
\textbf{Random ResNet-18 (BN)}
&
A standard ResNet-18 with randomly initialised and frozen weights. BatchNorm layers are retained, matching the architecture used by the trained-backbone baselines.
&
Tests whether the gains of CIRCLE can be explained simply by using a high-capacity convolutional architecture, even without training.
\\
\midrule

\textbf{Random ResNet-18 (LN)}
&
A randomly initialised and frozen ResNet-18 variant in which BatchNorm is replaced by LayerNorm.
&
Controls for possible pathologies of BatchNorm in randomly initialised networks. This separates the effect of the ResNet architecture from the effect of BatchNorm statistics.
\\
\midrule

\textbf{Random VGG-13}
&
A randomly initialised and frozen VGG-13-style convolutional network.
&
Provides a comparable-capacity convolutional baseline without residual connections or BatchNorm. Also tests if a different type of CNN could perform better.
\\
\midrule

\textbf{Patch-RNN}
&
The image is divided into patches, the patches are flattened, and the resulting sequence is processed by a single reservoir recurrent model. The final reservoir (hidden) state is used as the image embedding.
&
Tests a simple reservoir-based image representation obtained by sequentialising the image. This is the natural patch-level extension of earlier approaches that feed flattened image pixels into an ESN-style reservoir.
\\
\midrule

\textbf{Conv+\allowbreak{}Patch-RNN}
&
A small fixed random convolutional stack is first applied to the image. The resulting spatial feature map is then divided into patches and passed through a reservoir recurrent model.
&
Tests whether adding local random convolutional preprocessing is sufficient to close the gap between a simple patch reservoir and BiRC2D. This approach is strongly motivated by convolutional reservoir-computing approaches that combine fixed random CNN features with reservoir dynamics \cite{chang2019reinforcementlearningconvolutionalreservoir}.
\\
\midrule

\textbf{Scattering transform}
&
A deterministic wavelet-scattering feature extractor. It applies fixed wavelet filters, modulus nonlinearities, and averaging operations to produce a CNN-like representation without learned weights. Taken from the work by \citet{DBLP:journals/corr/abs-1203-1513}.
&
Tests CIRCLE against a strong non-random, but non-trained image representation whose filters are chosen from wavelet theory rather than learned from data. \\

\bottomrule
\end{tabular}
\end{table}

\subsection{Classification-head Ablations}\label{sec:appendix-headablation}

Table~\ref{tab:app-head-baselines} summarises the analytic classification heads used in the ablation study. All heads ignore classes that have not yet been observed. Head-specific hyperparameters, such as regularisation, shrinkage, and temperature, are tuned separately for each variant, together with the ensembling sizes. 

\begin{table}[ht]
\centering
\small
\caption{Classification-head baselines used in the ablation study. All heads consume the same BiRC2D features and are updated from streaming sufficient statistics.  Let \(z \in \mathbb{R}^d\) denote the frozen feature vector; \(n_c\) the number of observed examples from class \(c\); \(\mu_c\) the empirical class mean; and \(\pi_c\) the empirical class prior.}
\label{tab:app-head-baselines}
\begin{tabular}{p{0.09\linewidth}p{0.55\linewidth}p{0.27\linewidth}}
\toprule
\textbf{Head} & \textbf{Description} & \textbf{Main modelling assumption} \\
\midrule
Euclidean NCM
&
Nearest-class-mean classifier using Euclidean distance. For each class, the head stores the class sum and count, computes \(\mu_c\), and scores a test feature by
\[
s_c(z) = 2 z^\top \mu_c - \|\mu_c\|_2^2,
\]
which is equivalent to negative squared Euclidean distance up to a class-independent constant.
&
Each class is represented only by its mean. Class covariance and feature correlations are ignored. \newline Used by ADC \cite{goswami2024resurrectingoldclassesnew}.
\\
\midrule
Cosine NCM
&
Nearest-class-mean classifier using cosine similarity. Features are first \(\ell_2\)-normalised before being accumulated into class prototypes. The final prototype is normalised again, and the score is
\[
s_c(z) = \tau \,
\frac{z}{\|z\|_2}^{\top}
\frac{\mu_c}{\|\mu_c\|_2},
\]
where \(\tau\) is a temperature parameter.
&
Only angular similarity to the class prototype matters. Feature norms and covariance structure are ignored.
\\
\midrule

Ridge RLS
&
Ridge-regression classifier with one-hot class targets. The head accumulates
\[
A = \lambda I + \sum_i z_i z_i^\top,
\qquad
B = \sum_i z_i y_i^\top,
\]
and solves \(W = A^{-1}B\). Scores are \(s(z)=z^\top W\).
&
Learns a linear readout by least squares, with ridge regularisation. It does not explicitly model class-conditional covariance. Used by ACIL and DS-AL \cite{zhuang2022acilanalyticclassincrementallearning, zhuang2024dsaldualstreamanalyticlearning}.
\\
\midrule

Diagonal LDA
&
Shared-covariance LDA with a diagonal covariance approximation. The head stores class means and a pooled per-feature variance estimate. Scores use the diagonal precision:
\[
s_c(z)
=
z^\top \Sigma_{\mathrm{diag}}^{-1}\mu_c
-
\frac{1}{2}\mu_c^\top \Sigma_{\mathrm{diag}}^{-1}\mu_c
+
\log \pi_c .
\]
&
Models per-feature variances but ignores correlations between feature dimensions.
\\
\midrule

FeCAM Mahalanobis
&
FeCAM-style Mahalanobis classifier with class-specific covariance estimates. Each class is represented by a mean and a regularised covariance matrix. Scores are based on the negative Mahalanobis distance,
\[
s_c(z)
=
-\alpha
(z-\mu_c)^\top \Sigma_c^{-1}(z-\mu_c).
\]
&
Uses class-specific second-order structure, but the covariance estimates can be noisy when few examples are available per class. Used by FeCAM \cite{goswami2024fecamexploitingheterogeneityclass}.
\\
\midrule

QDA, class-specific covariance
&
Quadratic discriminant analysis with one covariance matrix per class. The head estimates \(\mu_c\) and \(\Sigma_c\), applies ridge regularisation and shrinkage toward the shared covariance, and scores by
\[
s_c(z)
=
-\frac{1}{2}
(z-\mu_c)^\top \Sigma_c^{-1}(z-\mu_c)
-
\frac{1}{2}\log |\Sigma_c|
+
\log \pi_c .
\]
&
Allows each class to have its own full covariance, giving a more flexible but more data-hungry classifier than SLDA. 
\\[0pt]

\bottomrule
\end{tabular}
\end{table}

The classification-head ablations test whether the gains come from the SLDA readout or simply from using any analytic streaming classifier. Mean-based heads test whether class means alone are sufficient. Ridge regression tests a discriminative linear readout. Diagonal LDA tests a cheaper covariance model than regular LDA. FeCAM-style Mahalanobis and QDA test more flexible class-specific covariance models. SLDA provides the best trade-off in this setting because it uses a full covariance estimate shared across classes, which captures feature correlations while avoiding the cold-start instability of estimating a separate covariance matrix for every class.

\clearpage
\subsection{Feature Dimensionality}
\label{sec:appendix-width}

CIRCLE concatenates features from multiple reservoirs, producing a large feature space for each classifier. We study the effect of feature dimensionality to see whether this feature ensembling gives CIRCLE an unfair advantage over the baselines. Unlike EFC++ and AdaGauss, which use 512 and 64 features on CIFAR-100, CIRCLE uses 8800 features per SLDA head on this dataset, and 7168 and 7200 on TinyImageNet and ImageNet-Subset, respectively. We therefore examine whether existing methods could also benefit from such large feature spaces.

Some baselines already use high-dimensional features. The frozen-backbone baselines ACIL and DS-AL use 8192 features on the first two datasets and 16384 on ImageNet-Subset (Appendix~\ref{sec:hyperparameters}). Hence, our main comparison already tests CIRCLE against \textit{frozen-backbone} methods with high-dimensional classifier inputs, which it outperforms at long horizons. We therefore focus the additional experiments on our strongest \textit{trained-backbone} baselines, EFC++ and AdaGauss, to see whether larger feature spaces could improve their performance. These experiments use CIFAR-100.

\paragraph{Feature dimension in trained-backbone methods.}
To control feature dimensionality, for EFC++ we change the number of channels in the final ResNet-18 residual stage. For AdaGauss, we change the dimension of the final bottleneck layer, keeping the backbone unchanged. We keep the remaining hyperparameters fixed to the values used in the main $T=20$ experiment, so they are not retuned for each dimension. Table~\ref{tab:baseline-width} reports mean and standard deviation over three random seeds. These experiments are run separately from the five-seed main comparison.

\begin{table}[!htb]
\centering
\small
\caption{Baseline feature-dimension experiments on CIFAR-100 at $T=20$. Accuracy values are percentages, reported as mean $\pm$ sample standard deviation over three seeds, rounded to one decimal place. ``Original'' denotes the width used in the main experiments.}
\label{tab:baseline-width}
\begin{tabular}{llcc}
\toprule
Method & Feature width & $\bar{A}$ & $A_T$ \\
\midrule
EFC++ & 64 & 43.8 $\pm$ 0.3 & 30.2 $\pm$ 0.1 \\
EFC++ & 128 & 48.7 $\pm$ 0.2 & 35.0 $\pm$ 0.3 \\
EFC++ & 256 & 52.4 $\pm$ 0.4 & 38.8 $\pm$ 0.0 \\
EFC++ & 512 (original) & 55.6 $\pm$ 0.5 & 41.9 $\pm$ 0.4 \\
EFC++ & 750 & 56.0 $\pm$ 0.2 & 42.0 $\pm$ 0.3 \\
EFC++ & 1024 & 56.3 $\pm$ 0.5 & 41.6 $\pm$ 0.5 \\
EFC++ & 1536 & 55.5 $\pm$ 0.3 & 40.8 $\pm$ 0.6 \\
\midrule
AdaGauss & 16 & 36.7 $\pm$ 0.1 & 21.8 $\pm$ 0.3 \\
AdaGauss & 32 & 49.5 $\pm$ 0.6 & 33.3 $\pm$ 0.8 \\
AdaGauss & 64 (original) & 56.1 $\pm$ 0.6 & 40.5 $\pm$ 0.6 \\
AdaGauss & 128 & 57.3 $\pm$ 0.5 & 40.4 $\pm$ 0.2 \\
AdaGauss & 256 & 52.0 $\pm$ 0.6 & 33.9 $\pm$ 0.5 \\
AdaGauss & 512 & 31.9 $\pm$ 0.5 & 19.6 $\pm$ 0.5 \\
\bottomrule
\end{tabular}
\end{table}

Table~\ref{tab:baseline-width} shows only limited gains from increasing feature dimensionality in EFC++ and AdaGauss. EFC++ improves average incremental accuracy by 0.7 percentage points from 512 to 1024 features, but its accuracy then decreases at 1536. AdaGauss performs best at 128 features and becomes substantially worse at larger dimensions. Thus, larger feature spaces do not consistently improve these trained-backbone methods.

In contrast, the middle panel of Figure~\ref{fig:ablation-ensembling} shows that CIRCLE benefits from increasing the feature space by concatenating more reservoir outputs. Together, these results support the fairness of our main comparison: it already includes high-dimensional frozen-backbone methods, while the trained-backbone baselines show little benefit from increasing dimensionality.

\section{Hyperparameters}
\label{sec:hyperparameters}

\subsection{Hyperparameter Tuning Process}
We tune hyperparameters separately for each dataset--split--algorithm tuple. That is, each algorithm on each dataset split is assigned its own hyperparameter configuration, selected using validation set performance. This is motivated by the fact that the different extreme splits (such as $T=10$ vs $T=100$) have very differing task dynamics (one split has very large tasks, another has very small tasks), and this likely means that very different hyperparameter sets would be optimal for both cases, hence we tune across splits. We use $20\%$ of data from each class for the validation dataset, on which the models are evaluated when selecting the results. After selection of the optimal hyperparameters, both the validation and train sets are merged into a new training set and the model is trained on the full set, and tested on the testing set. Hyperparameter tuning uses a held-out validation split for model selection only. After selecting hyperparameters, train and validation data are merged, and no test data are used for selection.

For the baseline methods, we use a few-stage randomized grid-search procedure. In the first stage, we tune the main optimization hyperparameters together, such as learning rate, number of training epochs, and batch size, while keeping the remaining hyperparameters fixed to sensible values, typically chosen from the corresponding literature or recommended defaults from the implementations. After this stage, we retain a small set of the best-performing configurations according to validation performance for every data--split--algorithm tuple. Starting from these configurations, we then tune the remaining method-specific hyperparameters using an additional randomized grid search. In the cases where a method still substantially underperforms relative to other baselines or previously reported results, we perform an additional, narrower search around the few best configurations found so far.

For our method, we follow the same general procedure, but separate the search into two stages. We first set the total ensemble size to one and tune the core hyperparameters of the model using randomized grid search. This allows us to identify strong single-model configurations without the additional cost of ensembling. After selecting the best-performing single-model configurations, we tune the ensemble-related hyperparameters, together with the reservoir's output dimensionality. Throughout this process, we always ensure that the total number (summed over trainable and non-trainable) parameters in our model does not differ by more than $5\%$ from a ResNet-18 that the baselines use, so that the model remains exactly comparable against the baselines (see the parameter counts in Table~\ref{tab:time}). Here too, the final hyperparameter configuration for each dataset--split--algorithm tuple is the one achieving the best validation performance.

\subsection{Selected Hyperparameters}

Below, we list the selected optimal hyperparameters for the main results in Table \ref{tab:main-comparison}. The remaining hyperparameters, for the ablation and other experiments can be directly found in our released codebase.
\paragraph{CIFAR-100, \(T=10\).}
\begin{enumerate}
\item \textbf{AdaGauss}: batch size 128; \(\alpha\) = 1; \(\lambda\) = 10; learning rate 0.1; adapter learning rate 0.01; backbone learning rate 0.01; epochs 440; \(s\) = 64; singular-value fraction 0.95; \(\tau\) = 2; weight decay 0.0005.
\item \textbf{EFC++}: batch size 128; balanced epochs 50; balanced learning rate 0.001; damping 0.1; first-task epochs 200; later-task epochs 200; \(\lambda\) = 10; prototype update -0.2.
\item \textbf{ADC}: batch size 256; epochs 800; initial epochs 800; initial learning rate 0.05; \(\lambda\) = 10; learning rate 0.05; temperature 2; weight decay 0.0005.
\item \textbf{ACIL}: incremental batch size 128; initial batch size 64; buffer size 8192; decay 0.18; \(\gamma\) = 0.28; initial epochs 258; initial learning rate 0.0408; initial weight decay 0.0005.
\item \textbf{FeCAM}: batch size 128; \(\alpha_1\) = 0.6; \(\alpha_2\) = 0.6; \(\beta\) = 0.5; initial epochs 400; initial learning rate 0.02; initial weight decay 0.0005.
\item \textbf{FeTRIL}: batch size 32; epochs 200; initial epochs 200; initial learning rate 0.02; initial weight decay 0.0005; learning rate 0.02; temperature 2.
\item \textbf{IL2A}: batch size 128; epochs 200; \(\gamma\) = 0.1; learning rate 0.005; augmentation ratio 2.5; step size 45; temperature 2; softmax temperature 0.1.
\item \textbf{PASS}: batch size 16; epochs 200; \(\gamma\) = 0.1; feature-distillation weight 10; prototype weight 10; learning rate 0.001; step size 45; temperature 2; softmax temperature 0.1; weight decay 0.0002.
\item \textbf{LwF}: batch size 128; epochs 220; initial epochs 220; initial learning rate 0.02; initial learning-rate decay 0.5; initial weight decay 0.0005; \(\lambda\) = 3; learning rate 0.02; learning-rate decay 0.5; temperature 2; weight decay 0.0005.
\item \textbf{DS-AL}: incremental batch size 128; initial batch size 64; buffer size 8192; compensation ratio 0.6; \(\gamma\) = 0.336; compensation \(\gamma\) = 0.1; initial epochs 258; initial learning rate 0.04; initial weight decay 0.0005; Scheduler: multi-step, milestones 120 and 140, with scheduler \(\gamma = 0.18\); warmup epochs 0.
\item \textbf{CIRCLE}: reservoir internal dimension 480; head ensemble size 8; feature ensemble size 8; reservoir output dimension 1100; CNN-stem channels 16 and 32; CNN-stem kernels 3 and 3; leak 0.8; leaky-ReLU slope 0.01; reservoir layers 1; patch sizes 2 and 4; sparsity 0.9; spectral radius 0.9.
\end{enumerate}

\paragraph{CIFAR-100, \(T=20\).}
\begin{enumerate}
\item \textbf{AdaGauss}: batch size 256; \(\alpha\) = 1; \(\lambda\) = 10; learning rate 0.1; adapter learning rate 0.01; backbone learning rate 0.01; epochs 220; \(s\) = 64; singular-value fraction 0.95; \(\tau\) = 2; weight decay 0.0005.
\item \textbf{EFC++}: batch size 64; balanced epochs 25; balanced learning rate 0.0002; damping 0.1; first-task epochs 200; later-task epochs 200; \(\lambda\) = 10; prototype update -0.2.
\item \textbf{ADC}: batch size 256; epochs 400; initial epochs 400; initial learning rate 0.05; \(\lambda\) = 10; learning rate 0.05; temperature 2; weight decay 0.0005.
\item \textbf{ACIL}: incremental batch size 128; initial batch size 128; buffer size 8192; decay 0.1; \(\gamma\) = 0.1; initial epochs 320; initial learning rate 0.02; initial weight decay 0.0005.
\item \textbf{FeCAM}: batch size 128; \(\alpha_1\) = 1; \(\alpha_2\) = 1; \(\beta\) = 0.5; initial epochs 400; initial learning rate 0.02; initial weight decay 0.0005.
\item \textbf{FeTRIL}: batch size 64; epochs 200; initial epochs 200; initial learning rate 0.02; initial weight decay 0.0005; learning rate 0.02; temperature 2.
\item \textbf{IL2A}: batch size 64; epochs 200; \(\gamma\) = 0.1; learning rate 0.001; augmentation ratio 2.5; step size 45; temperature 2; softmax temperature 0.1.
\item \textbf{PASS}: batch size 32; \(\gamma\) = 0.1; feature-distillation weight 10; prototype weight 10; learning rate 0.001; step size 45; temperature 2; softmax temperature 0.1; weight decay 0.0002.
\item \textbf{LwF}: batch size 128; epochs 110; initial epochs 110; initial learning rate 0.02; initial learning-rate decay 0.5; initial weight decay 0.0005; \(\lambda\) = 3; learning rate 0.02; learning-rate decay 0.5; temperature 2; weight decay 0.0005.
\item \textbf{DS-AL}: incremental batch size 128; initial batch size 128; buffer size 8192; compensation ratio 0.6; \(\gamma\) = 0.1; compensation \(\gamma\) = 0.1; initial epochs 320; initial learning rate 0.023; initial weight decay 0.0005; Scheduler: multi-step, milestones 120 and 140, with scheduler \(\gamma = 0.1\); warmup epochs 0.
\item \textbf{CIRCLE}: reservoir internal dimension 480; head ensemble size 8; feature ensemble size 8; reservoir output dimension 1100; CNN-stem channels 16 and 32; CNN-stem kernels 3 and 3; leak 0.8; leaky-ReLU slope 0.01; reservoir layers 1; patch sizes 2 and 4; sparsity 0.9; spectral radius 0.9.
\end{enumerate}

\paragraph{CIFAR-100, \(T=50\).}
\begin{enumerate}
\item \textbf{AdaGauss}: batch size 128; \(\alpha\) = 1; \(\lambda\) = 10; learning rate 0.1; adapter learning rate 0.005; backbone learning rate 0.005; epochs 220; \(s\) = 64; singular-value fraction 0.95; \(\tau\) = 2; weight decay 0.0005.
\item \textbf{EFC++}: batch size 16; balanced epochs 19; balanced learning rate 0.0002; damping 0.1; first-task epochs 150; later-task epochs 150; \(\lambda\) = 10; prototype update -0.2.
\item \textbf{ADC}: batch size 64; epochs 300; initial epochs 300; initial learning rate 0.025; \(\lambda\) = 10; learning rate 0.025; temperature 2; weight decay 0.0005.
\item \textbf{ACIL}: incremental batch size 64; initial batch size 64; buffer size 8192; decay 0.18; \(\gamma\) = 0.28; initial epochs 218; initial learning rate 0.0432; initial weight decay 0.0005.
\item \textbf{FeCAM}: batch size 64; \(\alpha_1\) = 1; \(\alpha_2\) = 1; \(\beta\) = 0.7; initial epochs 272; initial learning rate 0.036; initial weight decay 0.0005.
\item \textbf{FeTRIL}: batch size 32; epochs 200; initial epochs 200; initial learning rate 0.005; initial weight decay 0.0005; learning rate 0.005; temperature 2.
\item \textbf{IL2A}: batch size 64; epochs 150; \(\gamma\) = 0.1; learning rate 0.00025; augmentation ratio 2.5; step size 45; temperature 2; softmax temperature 0.1.
\item \textbf{PASS}: batch size 32; epochs 75; \(\gamma\) = 0.1; feature-distillation weight 10; prototype weight 10; learning rate 0.00075; step size 45; temperature 2; softmax temperature 0.1; weight decay 0.0002.
\item \textbf{LwF}: batch size 64; epochs 110; initial epochs 110; initial learning rate 0.015; initial learning-rate decay 0.25; initial weight decay 0.0005; \(\lambda\) = 3; learning rate 0.015; learning-rate decay 0.25; temperature 2; weight decay 0.0005.
\item \textbf{DS-AL}: incremental batch size 64; initial batch size 64; buffer size 8192; compensation ratio 0.6; \(\gamma\) = 0.336; compensation \(\gamma\) = 0.1; initial epochs 218; initial learning rate 0.0432; initial weight decay 0.0005; Scheduler: multi-step, milestones 120 and 140, with scheduler \(\gamma = 0.18\); warmup epochs 0.
\item \textbf{CIRCLE}: reservoir internal dimension 480; head ensemble size 8; feature ensemble size 8; reservoir output dimension 1100; CNN-stem channels 16 and 32; CNN-stem kernels 3 and 3; leak 0.8; leaky-ReLU slope 0.01; reservoir layers 1; patch sizes 2 and 4; sparsity 0.9; spectral radius 0.9.
\end{enumerate}

\paragraph{CIFAR-100, \(T=100\).}
\begin{enumerate}
\item \textbf{AdaGauss}: batch size 96; \(\alpha\) = 1; \(\lambda\) = 2; learning rate 0.125; adapter learning rate 0.01; backbone learning rate 0.01; epochs 55; \(s\) = 64; singular-value fraction 0.95; \(\tau\) = 2.5; weight decay 0.0005.
\item \textbf{EFC++}: batch size 16; balanced epochs 24; balanced learning rate 0.00032; damping 0.05; first-task epochs 172; later-task epochs 112; \(\lambda\) = 8; prototype update -0.1.
\item \textbf{ADC}: batch size 256; epochs 360; initial epochs 360; initial learning rate 0.08; \(\lambda\) = 6; learning rate 0.08; temperature 2; weight decay 0.0005.
\item \textbf{ACIL}: incremental batch size 64; initial batch size 32; buffer size 8192; decay 0.108; \(\gamma\) = 0.84; initial epochs 218; initial learning rate 0.0432; initial weight decay 0.0005.
\item \textbf{FeCAM}: batch size 64; \(\alpha_1\) = 2; \(\alpha_2\) = 0.7; \(\beta\) = 1.05; initial epochs 231; initial learning rate 0.018; initial weight decay 0.0005.
\item \textbf{FeTRIL}: batch size 16; epochs 150; initial epochs 260; initial learning rate 0.0035; initial weight decay 0.0005; learning rate 0.0025; temperature 2.5.
\item \textbf{IL2A}: batch size 32; epochs 150; \(\gamma\) = 0.16; learning rate 0.000175; augmentation ratio 1.5; step size 45; temperature 0.14; softmax temperature 0.14.
\item \textbf{PASS}: batch size 32; epochs 75; \(\gamma\) = 0.1; feature-distillation weight 10; prototype weight 10; learning rate 0.00075; step size 45; temperature 2; softmax temperature 0.1; weight decay 0.0002.
\item \textbf{LwF}: batch size 128; epochs 110; initial epochs 110; initial learning rate 0.02; initial learning-rate decay 0.5; initial weight decay 0.0005; \(\lambda\) = 3; learning rate 0.02; learning-rate decay 0.5; temperature 2; weight decay 0.0005.
\item \textbf{DS-AL}: incremental batch size 64; initial batch size 32; buffer size 8192; compensation ratio 0.6; \(\gamma\) = 1.008; compensation \(\gamma\) = 0.1; initial epochs 218; initial learning rate 0.0432; initial weight decay 0.0005; Scheduler: multi-step, milestones 120 and 140, with scheduler \(\gamma = 0.108\); warmup epochs 0.
\item \textbf{CIRCLE}: reservoir internal dimension 480; head ensemble size 8; feature ensemble size 8; reservoir output dimension 1100; CNN-stem channels 16 and 32; CNN-stem kernels 3 and 3; leak 0.8; leaky-ReLU slope 0.01; reservoir layers 1; patch sizes 2 and 4; sparsity 0.9; spectral radius 0.9.
\end{enumerate}

\paragraph{TinyImageNet, \(T=10\).}
\begin{enumerate}
\item \textbf{AdaGauss}: batch size 192; \(\alpha\) = 0.6; \(\lambda\) = 8; learning rate 0.1; adapter learning rate 0.01; backbone learning rate 0.01; epochs 110; \(s\) = 64; singular-value fraction 0.95; \(\tau\) = 2; weight decay 0.0005.
\item \textbf{EFC++}: batch size 80; balanced epochs 62; balanced learning rate 0.00075; damping 0.05; first-task epochs 75; later-task epochs 90; \(\lambda\) = 3; prototype update -0.1.
\item \textbf{ADC}: batch size 256; epochs 320; initial epochs 320; initial learning rate 0.08; \(\lambda\) = 15; learning rate 0.08; temperature 3; weight decay 0.0005.
\item \textbf{ACIL}: incremental batch size 192; initial batch size 256; buffer size 8192; backbone ResNet-18 CBAM; decay 0.009; \(\gamma\) = 0.048; initial epochs 239; initial learning rate 0.08925; initial weight decay 0.0005; learning-rate milestones 143 and 203; scheduler MultiStep; warmup epochs 0.
\item \textbf{FeCAM}: batch size 256; \(\alpha_1\) = 1; \(\alpha_2\) = 0.6; \(\beta\) = 0.3; initial epochs 200; initial learning rate 0.5; initial weight decay 0.0005.
\item \textbf{FeTRIL}: batch size 16; epochs 50; initial epochs 300; initial learning rate 0.1; initial weight decay 0.0002; learning rate 0.0005; temperature 2; weight decay 0.001.
\item \textbf{IL2A}: batch size 64; epochs 100; \(\gamma\) = 0.3; feature-distillation weight 1.5; prototype weight 0.5; learning rate 0.0005; augmentation ratio 1; step size 45; temperature 0.8; softmax temperature 0.8; weight decay 0.0002.
\item \textbf{PASS}: batch size 32; epochs 100; \(\gamma\) = 0.1; feature-distillation weight 15; prototype weight 0.5; learning rate 0.001; step size 45; temperature 0.05; softmax temperature 0.05; weight decay 0.0002.
\item \textbf{LwF}: batch size 64; epochs 320; initial epochs 320; initial learning rate 0.01; initial learning-rate decay 0.5; initial weight decay 0.0005; \(\lambda\) = 3; learning rate 0.01; learning-rate decay 0.25; temperature 2.5; weight decay 0.0005.
\item \textbf{DS-AL}: incremental batch size 512; initial batch size 128; buffer size 8192; compensation ratio 1; \(\gamma\) = 0.06; compensation \(\gamma\) = 0.13; initial epochs 260; initial learning rate 0.108375; initial weight decay 0.0002; Scheduler: multi-step, milestones 156 and 221, with scheduler \(\gamma = 0.009\); warmup epochs 5.
\item \textbf{CIRCLE}: reservoir internal dimension 384; head ensemble size 7; feature ensemble size 7; reservoir output dimension 1024; CNN-stem channels 16; CNN-stem kernels 3; leak 0.7; leaky-ReLU slope 0.01; maximum rotation 0; reservoir layers 2; patch sizes 2; sparsity 0.5; spectral radius 0.9.
\end{enumerate}

\paragraph{TinyImageNet, \(T=20\).}
\begin{enumerate}
\item \textbf{AdaGauss}: batch size 256; \(\alpha\) = 1; \(\lambda\) = 10; learning rate 0.1; adapter learning rate 0.01; backbone learning rate 0.01; epochs 110; \(s\) = 64; singular-value fraction 0.95; \(\tau\) = 2; weight decay 0.0005.
\item \textbf{EFC++}: batch size 128; balanced epochs 100; balanced learning rate 0.001; damping 0.1; first-task epochs 100; later-task epochs 100; \(\lambda\) = 10; prototype update -0.2.
\item \textbf{ADC}: batch size 64; epochs 400; initial epochs 400; initial learning rate 0.005; \(\lambda\) = 10; learning rate 0.005; temperature 2; weight decay 0.0005.
\item \textbf{ACIL}: incremental batch size 64; initial batch size 192; buffer size 8192; backbone ResNet-18 CBAM; decay 0.05; \(\gamma\) = 0.18; initial epochs 180; initial learning rate 0.006; initial weight decay 0.0002; learning-rate milestones 108 and 153; scheduler MultiStep; warmup epochs 10.
\item \textbf{FeCAM}: batch size 128; \(\alpha_1\) = 1; \(\alpha_2\) = 1; \(\beta\) = 0.5; initial epochs 200; initial learning rate 0.5; initial weight decay 0.0005.
\item \textbf{FeTRIL}: batch size 16; epochs 50; initial epochs 300; initial learning rate 0.1; initial weight decay 0.0002; learning rate 0.0005; temperature 2; weight decay 0.001.
\item \textbf{IL2A}: batch size 64; epochs 100; \(\gamma\) = 0.3; feature-distillation weight 1.5; prototype weight 0.5; learning rate 0.0005; augmentation ratio 1; step size 45; temperature 0.8; softmax temperature 0.8; weight decay 0.0002.
\item \textbf{PASS}: batch size 32; epochs 100; \(\gamma\) = 0.1; feature-distillation weight 15; prototype weight 0.5; learning rate 0.001; step size 45; temperature 0.05; softmax temperature 0.05; weight decay 0.0002.
\item \textbf{LwF}: batch size 128; epochs 440; initial epochs 440; initial learning rate 0.02; initial learning-rate decay 0.5; initial weight decay 0.0005; \(\lambda\) = 3; learning rate 0.02; learning-rate decay 0.5; temperature 2; weight decay 0.0005.
\item \textbf{DS-AL}: incremental batch size 160; initial batch size 128; buffer size 8192; compensation ratio 0.75; \(\gamma\) = 0.18; compensation \(\gamma\) = 0.13; initial epochs 180; initial learning rate 0.009; initial weight decay 0.0002; Scheduler: multi-step, milestones 108 and 153, with scheduler \(\gamma = 0.05\); warmup epochs 10.
\item \textbf{CIRCLE}: reservoir internal dimension 384; head ensemble size 7; feature ensemble size 7; reservoir output dimension 1024; CNN-stem channels 16; CNN-stem kernels 3; leak 0.7; leaky-ReLU slope 0.01; maximum rotation 0; reservoir layers 2; patch sizes 2; sparsity 0.5; spectral radius 0.9.
\end{enumerate}

\paragraph{TinyImageNet, \(T=50\).}
\begin{enumerate}
\item \textbf{AdaGauss}: batch size 192; \(\alpha\) = 1.2; \(\lambda\) = 5; learning rate 0.08; adapter learning rate 0.01; backbone learning rate 0.01; epochs 220; \(s\) = 64; singular-value fraction 0.95; \(\tau\) = 1.6; weight decay 0.0005.
\item \textbf{EFC++}: batch size 64; balanced epochs 12; balanced learning rate 0.001; damping 0.1; first-task epochs 100; later-task epochs 100; \(\lambda\) = 10; prototype update -0.2.
\item \textbf{ADC}: batch size 128; epochs 100; initial epochs 100; initial learning rate 0.06; \(\lambda\) = 60; learning rate 0.06; temperature 4; weight decay 0.0005.
\item \textbf{ACIL}: incremental batch size 96; initial batch size 384; buffer size 8192; backbone ResNet-18 CBAM; decay 0.18; \(\gamma\) = 0.18; initial epochs 220; initial learning rate 0.015; initial weight decay 0.0002; learning-rate milestones 132 and 187; scheduler MultiStep; warmup epochs 10.
\item \textbf{FeCAM}: batch size 32; \(\alpha_1\) = 1; \(\alpha_2\) = 0.7; \(\beta\) = 0.7; initial epochs 190; initial learning rate 0.027; initial weight decay 0.0005.
\item \textbf{FeTRIL}: batch size 32; epochs 200; initial epochs 200; initial learning rate 0.005; initial weight decay 0.0002; learning rate 0.005; temperature 2; weight decay 0.001.
\item \textbf{IL2A}: batch size 96; epochs 112; \(\gamma\) = 0.05; learning rate 0.0005; augmentation ratio 1.5; step size 45; temperature 0.3; softmax temperature 0.3.
\item \textbf{PASS}: batch size 48; epochs 56; \(\gamma\) = 0.125; feature-distillation weight 15; prototype weight 7.5; learning rate 0.0009; step size 45; temperature 0.14; softmax temperature 0.14; weight decay 0.0002.
\item \textbf{LwF}: batch size 192; epochs 320; initial epochs 320; initial learning rate 0.01; initial learning-rate decay 0.5; initial weight decay 0.0005; \(\lambda\) = 4.2; learning rate 0.01; learning-rate decay 0.75; temperature 3.2; weight decay 0.0005.
\item \textbf{DS-AL}: incremental batch size 256; initial batch size 256; buffer size 8192; compensation ratio 0.85; \(\gamma\) = 0.252; compensation \(\gamma\) = 0.1; initial epochs 139; initial learning rate 0.0918; initial weight decay 0.0005; Scheduler: multi-step, milestones 45 and 90, with scheduler \(\gamma = 0.294\); warmup epochs 10.
\item \textbf{CIRCLE}: reservoir internal dimension 384; head ensemble size 7; feature ensemble size 7; reservoir output dimension 1024; CNN-stem channels 16; CNN-stem kernels 3; leak 0.7; leaky-ReLU slope 0.01; maximum rotation 0; reservoir layers 2; patch sizes 2; sparsity 0.5; spectral radius 0.9.
\end{enumerate}

\paragraph{TinyImageNet, \(T=100\).}
\begin{enumerate}
\item \textbf{AdaGauss}: batch size 192; \(\alpha\) = 0.72; \(\lambda\) = 1.5; learning rate 0.1; adapter learning rate 0.01; backbone learning rate 0.01; epochs 220; \(s\) = 64; singular-value fraction 0.95; \(\tau\) = 3.2; weight decay 0.0005.
\item \textbf{EFC++}: batch size 48; balanced epochs 10; balanced learning rate 0.00125; damping 0.1; first-task epochs 90; later-task epochs 115; \(\lambda\) = 3; prototype update -0.16.
\item \textbf{ADC}: batch size 32; epochs 360; initial epochs 360; initial learning rate 0.005; \(\lambda\) = 22; learning rate 0.005; temperature 1; weight decay 0.0005.
\item \textbf{ACIL}: incremental batch size 384; initial batch size 384; buffer size 8192; backbone ResNet-18 CBAM; decay 0.392; \(\gamma\) = 0.504; initial epochs 154; initial learning rate 0.054; initial weight decay 0.0005; learning-rate milestones 69 and 108; scheduler MultiStep; warmup epochs 5.
\item \textbf{FeCAM}: batch size 32; \(\alpha_1\) = 1; \(\alpha_2\) = 0.7; \(\beta\) = 0.7; initial epochs 190; initial learning rate 0.027; initial weight decay 0.0005.
\item \textbf{FeTRIL}: batch size 32; epochs 200; initial epochs 200; initial learning rate 0.01; initial weight decay 0.0002; learning rate 0.0075; temperature 2.5; weight decay 0.001.
\item \textbf{IL2A}: batch size 96; epochs 112; \(\gamma\) = 0.05; learning rate 0.0005; augmentation ratio 1.5; step size 45; temperature 0.3; softmax temperature 0.3; weight decay 0.0002.
\item \textbf{PASS}: batch size 32; epochs 100; \(\gamma\) = 0.2; feature-distillation weight 60; prototype weight 2; learning rate 0.0005; step size 45; temperature 0.1; softmax temperature 0.1; weight decay 0.0002.
\item \textbf{LwF}: batch size 64; epochs 308; initial epochs 308; initial learning rate 0.02; initial learning-rate decay 0.5; initial weight decay 0.0005; \(\lambda\) = 3; learning rate 0.02; learning-rate decay 0.25; temperature 1; weight decay 0.0005.
\item \textbf{DS-AL}: incremental batch size 512; initial batch size 192; buffer size 8192; compensation ratio 1; \(\gamma\) = 0.63; compensation \(\gamma\) = 0.1; initial epochs 169; initial learning rate 0.054; initial weight decay 0.0005; Scheduler: multi-step, milestones 101 and 144, with scheduler \(\gamma = 0.588\); warmup epochs 0.
\item \textbf{CIRCLE}: reservoir internal dimension 384; head ensemble size 7; feature ensemble size 7; reservoir output dimension 1024; CNN-stem channels 16; CNN-stem kernels 3; leak 0.7; leaky-ReLU slope 0.01; maximum rotation 0; reservoir layers 2; patch sizes 2; sparsity 0.5; spectral radius 0.9.
\end{enumerate}

\paragraph{ImageNet-Subset, \(T=10\).}
\begin{enumerate}
\item \textbf{AdaGauss}: batch size 128; \(\alpha\) = 1; \(\lambda\) = 10; learning rate 0.1; adapter learning rate 0.05; backbone learning rate 0.05; epochs 110; \(s\) = 64; singular-value fraction 0.95; \(\tau\) = 2; 224-pixel inputs enabled; weight decay 0.0005.
\item \textbf{EFC++}: batch size 64; balanced epochs 50; balanced learning rate 0.001; damping 0.1; first-task epochs 100; later-task epochs 100; \(\lambda\) = 10; prototype update -0.2.
\item \textbf{ADC}: batch size 256; epochs 200; initial epochs 200; initial learning rate 0.25; \(\lambda\) = 20; learning rate 0.25; temperature 2; weight decay 0.0005.
\item \textbf{ACIL}: incremental batch size 256; initial batch size 128; buffer size 16384; decay 0.28; \(\gamma\) = 0.28; initial epochs 145; initial learning rate 0.0408; initial weight decay 0.0005.
\item \textbf{FeCAM}: batch size 256; \(\alpha_1\) = 1; \(\alpha_2\) = 0.6; \(\beta\) = 0.3; initial epochs 200; initial learning rate 0.5; initial weight decay 0.0005.
\item \textbf{FeTRIL}: batch size 256; epochs 400; initial epochs 400; initial learning rate 0.02; initial weight decay 0.0005; learning rate 0.02; temperature 2.
\item \textbf{IL2A}: batch size 64; epochs 100; \(\gamma\) = 0.1; learning rate 0.001; augmentation ratio 2.5; step size 45; temperature 2; softmax temperature 0.1.
\item \textbf{PASS}: batch size 32; epochs 100; \(\gamma\) = 0.1; feature-distillation weight 10; prototype weight 10; learning rate 0.0002; step size 45; temperature 2; softmax temperature 0.1; weight decay 0.0002.
\item \textbf{LwF}: batch size 256; epochs 110; initial epochs 110; initial learning rate 0.02; initial learning-rate decay 0.02; initial weight decay 0.0005; \(\lambda\) = 3; learning rate 0.02; learning-rate decay 0.02; temperature 2; weight decay 0.0005.
\item \textbf{DS-AL}: incremental batch size 256; initial batch size 128; buffer size 16384; compensation ratio 1.5; \(\gamma\) = 0.28; compensation \(\gamma\) = 0.1; initial epochs 145; initial learning rate 0.03468; initial weight decay 0.0005; Scheduler: multi-step, milestones 30 and 60, with scheduler \(\gamma = 0.28\); warmup epochs 0.
\item \textbf{CIRCLE}: reservoir internal dimension 460; head ensemble size 9; feature ensemble size 8; reservoir output dimension 900; CNN-stem activation leaky ReLU; CNN-stem channels 4, 6, and 6; CNN-stem kernels 3, 3, and 3; internal state width 344; leak 0.8; leaky-ReLU slope 0.01; reservoir layers 1; patch sizes 7, 8, and 16; sparsity 0.9; spectral radius 0.9.
\end{enumerate}

\paragraph{ImageNet-Subset, \(T=20\).}
\begin{enumerate}
\item \textbf{AdaGauss}: batch size 128; \(\alpha\) = 1; \(\lambda\) = 10; learning rate 0.1; adapter learning rate 0.01; backbone learning rate 0.01; epochs 220; \(s\) = 64; singular-value fraction 0.95; \(\tau\) = 2; 224-pixel inputs enabled; weight decay 0.0005.
\item \textbf{EFC++}: batch size 64; balanced epochs 50; balanced learning rate 0.001; damping 0.1; first-task epochs 100; later-task epochs 100; \(\lambda\) = 10; prototype update -0.2.
\item \textbf{ADC}: batch size 128; epochs 400; initial epochs 400; initial learning rate 0.05; \(\lambda\) = 20; learning rate 0.05; temperature 2; weight decay 0.0005.
\item \textbf{ACIL}: incremental batch size 256; initial batch size 256; buffer size 16384; decay 0.1; \(\gamma\) = 0.1; initial epochs 180; initial learning rate 0.02; initial weight decay 0.0005.
\item \textbf{FeCAM}: batch size 64; \(\alpha_1\) = 1; \(\alpha_2\) = 1; \(\beta\) = 0.5; initial epochs 400; initial learning rate 0.1; initial weight decay 0.0005.
\item \textbf{FeTRIL}: batch size 256; epochs 200; initial epochs 200; initial learning rate 0.02; initial weight decay 0.0005; learning rate 0.02; temperature 2.
\item \textbf{IL2A}: batch size 128; epochs 200; \(\gamma\) = 0.1; learning rate 0.001; augmentation ratio 2.5; step size 45; temperature 2; softmax temperature 0.1.
\item \textbf{PASS}: batch size 32; \(\gamma\) = 0.1; feature-distillation weight 10; prototype weight 10; learning rate 0.0002; step size 45; temperature 2; softmax temperature 0.1; weight decay 0.0002.
\item \textbf{LwF}: batch size 256; epochs 110; initial epochs 110; initial learning rate 0.02; initial learning-rate decay 0.02; initial weight decay 0.0005; \(\lambda\) = 3; learning rate 0.02; learning-rate decay 0.02; temperature 2; weight decay 0.0005.
\item \textbf{DS-AL}: incremental batch size 256; initial batch size 256; buffer size 16384; compensation ratio 1.5; \(\gamma\) = 0.12; compensation \(\gamma\) = 0.1; initial epochs 180; initial learning rate 0.02; initial weight decay 0.0005; Scheduler: multi-step, milestones 30 and 60, with scheduler \(\gamma = 0.1\); warmup epochs 0.
\item \textbf{CIRCLE}: reservoir internal dimension 460; head ensemble size 9; feature ensemble size 8; reservoir output dimension 900; CNN-stem activation leaky ReLU; CNN-stem channels 4, 6, and 6; CNN-stem kernels 3, 3, and 3; internal state width 344; leak 0.8; leaky-ReLU slope 0.01; reservoir layers 1; patch sizes 7, 8, and 16; sparsity 0.9; spectral radius 0.9.
\end{enumerate}

\paragraph{ImageNet-Subset, \(T=50\).}
\begin{enumerate}
\item \textbf{AdaGauss}: batch size 128; \(\alpha\) = 1; \(\lambda\) = 10; learning rate 0.075; adapter learning rate 0.0075; backbone learning rate 0.0075; epochs 55; \(s\) = 64; singular-value fraction 0.95; \(\tau\) = 2; 224-pixel inputs enabled; weight decay 0.0005.
\item \textbf{EFC++}: batch size 64; balanced epochs 12; balanced learning rate 0.001; damping 0.1; first-task epochs 100; later-task epochs 100; \(\lambda\) = 10; prototype update -0.2.
\item \textbf{ADC}: batch size 96; epochs 100; initial epochs 100; initial learning rate 0.0375; \(\lambda\) = 20; learning rate 0.0375; temperature 2; weight decay 0.0005.
\item \textbf{ACIL}: incremental batch size 384; initial batch size 128; buffer size 16384; decay 0.28; \(\gamma\) = 0.28; initial epochs 122; initial learning rate 0.0432; initial weight decay 0.0005.
\item \textbf{FeCAM}: batch size 384; \(\alpha_1\) = 1; \(\alpha_2\) = 1; \(\beta\) = 0.5; initial epochs 204; initial learning rate 0.3; initial weight decay 0.0005.
\item \textbf{FeTRIL}: batch size 128; epochs 150; initial epochs 150; initial learning rate 0.015; initial weight decay 0.0005; learning rate 0.015; temperature 2.
\item \textbf{IL2A}: batch size 128; epochs 50; \(\gamma\) = 0.1; learning rate 0.00025; augmentation ratio 2.5; step size 45; temperature 2; softmax temperature 0.1.
\item \textbf{PASS}: batch size 16; epochs 50; \(\gamma\) = 0.1; feature-distillation weight 10; prototype weight 10; learning rate 0.0002; step size 45; temperature 2; softmax temperature 0.1; weight decay 0.0002.
\item \textbf{LwF}: batch size 128; epochs 82; initial epochs 82; initial learning rate 0.005; initial learning-rate decay 0.02; initial weight decay 0.0005; \(\lambda\) = 3; learning rate 0.005; learning-rate decay 0.02; temperature 2; weight decay 0.0005.
\item \textbf{DS-AL}: incremental batch size 384; initial batch size 128; buffer size 16384; compensation ratio 1.5; \(\gamma\) = 0.28; compensation \(\gamma\) = 0.1; initial epochs 122; initial learning rate 0.03672; initial weight decay 0.0005; Scheduler: multi-step, milestones 30 and 60, with scheduler \(\gamma = 0.28\); warmup epochs 0.
\item \textbf{CIRCLE}: reservoir internal dimension 460; head ensemble size 9; feature ensemble size 8; reservoir output dimension 900; CNN-stem activation leaky ReLU; CNN-stem channels 4, 6, and 6; CNN-stem kernels 3, 3, and 3; internal state width 344; leak 0.8; leaky-ReLU slope 0.01; reservoir layers 1; patch sizes 7, 8, and 16; sparsity 0.9; spectral radius 0.9.
\end{enumerate}

\paragraph{ImageNet-Subset, \(T=100\).}
\begin{enumerate}
\item \textbf{AdaGauss}: batch size 128; \(\alpha\) = 1; \(\lambda\) = 10; learning rate 0.075; adapter learning rate 0.01; backbone learning rate 0.01; epochs 55; \(s\) = 64; singular-value fraction 0.95; \(\tau\) = 2; 224-pixel inputs enabled; weight decay 0.0005.
\item \textbf{EFC++}: batch size 96; balanced epochs 40; balanced learning rate 0.00125; damping 0.16; first-task epochs 75; later-task epochs 75; \(\lambda\) = 8; prototype update -0.25.
\item \textbf{ADC}: batch size 128; epochs 115; initial epochs 115; initial learning rate 0.028125; \(\lambda\) = 30; learning rate 0.028125; temperature 1; weight decay 0.0005.
\item \textbf{ACIL}: incremental batch size 256; initial batch size 64; buffer size 16384; decay 0.28; \(\gamma\) = 0.392; initial epochs 110; initial learning rate 0.03024; initial weight decay 0.0005.
\item \textbf{FeCAM}: batch size 64; \(\alpha_1\) = 1.5; \(\alpha_2\) = 1.2; \(\beta\) = 0.6; initial epochs 440; initial learning rate 0.2; initial weight decay 0.0005.
\item \textbf{FeTRIL}: batch size 128; epochs 100; initial epochs 230; initial learning rate 0.017; initial weight decay 0.0005; learning rate 0.03; temperature 1.
\item \textbf{IL2A}: batch size 128; epochs 180; \(\gamma\) = 0.1; learning rate 0.0005; augmentation ratio 2.5; step size 45; temperature 0.3; softmax temperature 0.3.
\item \textbf{PASS}: batch size 48; \(\gamma\) = 0.16; feature-distillation weight 15; prototype weight 1.5; learning rate 0.0003; step size 45; temperature 0.1; softmax temperature 0.1; weight decay 0.0002.
\item \textbf{LwF}: batch size 96; epochs 70; initial epochs 70; initial learning rate 0.006; initial learning-rate decay 0.01; initial weight decay 0.0005; \(\lambda\) = 6; learning rate 0.006; learning-rate decay 0.02; temperature 4; weight decay 0.0005.
\item \textbf{DS-AL}: incremental batch size 256; initial batch size 64; buffer size 16384; compensation ratio 1.25; \(\gamma\) = 0.392; compensation \(\gamma\) = 0.1; initial epochs 110; initial learning rate 0.03024; initial weight decay 0.0005; Scheduler: multi-step, milestones 30 and 60, with scheduler \(\gamma = 0.28\); warmup epochs 0.
\item \textbf{CIRCLE}: reservoir internal dimension 460; head ensemble size 9; feature ensemble size 8; reservoir output dimension 900; CNN-stem activation leaky ReLU; CNN-stem channels 4, 6, and 6; CNN-stem kernels 3, 3, and 3; internal state width 344; leak 0.8; leaky-ReLU slope 0.01; reservoir layers 1; patch sizes 7, 8, and 16; sparsity 0.9; spectral radius 0.9.
\end{enumerate}

\end{document}